\documentclass{article}

      \usepackage[nonatbib,final]{neurips_2020}

\usepackage[utf8]{inputenc} 
\usepackage[T1]{fontenc}    
\usepackage{hyperref}       
\usepackage{url}            
\usepackage{booktabs}       
\usepackage{amsfonts}       
\usepackage{nicefrac}       
\usepackage{microtype}      

\usepackage{amsmath,amsfonts,bm}

\def\eqref#1{equation~\ref{#1}}

\def\1{\bm{1}}

\DeclareMathAlphabet{\mathsfit}{\encodingdefault}{\sfdefault}{m}{sl}
\SetMathAlphabet{\mathsfit}{bold}{\encodingdefault}{\sfdefault}{bx}{n}

\newcommand{\E}{\mathbb{E}}

\usepackage{hyperref}
\usepackage{url}

\newcommand\bertbase{BERT$_{\small \textsc{BASE}}$\xspace}
\newcommand\bertlarge{BERT$_{\small \textsc{LARGE}}$\xspace}
\newcommand\bertft{BERT$_{\small \textsc{FINETUNE}}$\xspace}

\usepackage{xcolor}

\newcommand{\divg}[3][]{\mathcal{D}_{#1}\left({#2} \;\middle\|\; {#3} \right)}

\newcommand{\cmark}{\ding{51}}%
\newcommand{\xmark}{\ding{55}}%

\usepackage{url}

\usepackage{graphicx}
\usepackage{caption}
\usepackage{booktabs}
\usepackage{enumitem}
\usepackage{xspace}
\usepackage{multirow}
\usepackage{algorithm}
\usepackage{algorithmic}

\usepackage{amssymb}
\usepackage{pifont}
\usepackage{amsmath}
\usepackage{amsfonts}
\usepackage{amsthm}
\usepackage{hyperref}
\newtheorem{theorem}{Theorem}

\newtheorem{innercustomthm}{Theorem}
\newenvironment{customthm}[1]
  {\renewcommand\theinnercustomthm{#1}\innercustomthm}
  {\endinnercustomthm}

\usepackage{subcaption}

\def\name{UDA\xspace}

\def\jft{JFT~\cite{hinton2015distilling, chollet2017xception}\xspace}

\title{Unsupervised Data Augmentation\\ for Consistency Training}

\author{Qizhe Xie$^{1,2}$, Zihang Dai$^{1,2}$, Eduard Hovy$^2$, Minh-Thang Luong$^1$, Quoc V. Le$^1$\\
$^1$ Google Research, Brain Team, $^2$ Carnegie Mellon University \\
  \texttt{\{qizhex, dzihang, hovy\}@cs.cmu.edu, \{thangluong, qvl\}@google.com} \\
}

\begin{document}

\maketitle

\begin{abstract}
Semi-supervised learning lately has shown much promise in improving deep learning models when labeled data is scarce. Common among recent approaches is the use of consistency training on a large amount of unlabeled data to constrain model predictions to be invariant to input noise. In this work, we present a new perspective on how to effectively noise unlabeled examples and argue that the quality of noising, specifically those produced by advanced data augmentation methods, plays a crucial role in semi-supervised learning. By substituting simple noising operations with advanced data augmentation methods such as RandAugment and back-translation, our method brings substantial improvements across six language and three vision tasks under the same consistency training framework. On the IMDb text classification dataset, with only 20 labeled examples, our method achieves an error rate of 4.20, outperforming the state-of-the-art model trained on 25,000 labeled examples. On a standard semi-supervised learning benchmark, CIFAR-10, our method outperforms all previous approaches and achieves an error rate of 5.43 with only 250 examples. Our method also combines well with transfer learning, e.g., when finetuning from BERT, and yields improvements in high-data regime, such as ImageNet, whether when there is only 10\% labeled data or when a full labeled set with 1.3M extra unlabeled examples is used.\footnote{Code is available at \url{https://github.com/google-research/uda}.}
\end{abstract}

\vspace{-0.5em}
\section{Introduction}
\vspace{-0.5em}
\label{sec:intro}

A fundamental weakness of deep learning is that it typically requires a lot of labeled data to work well. 
Semi-supervised learning (SSL) ~\cite{chapelle2009semi} is one of the most promising paradigms of leveraging unlabeled data to address this weakness. The recent works in SSL are diverse but those that are based on consistency training~\cite{bachman2014learning,rasmus2015semi,laine2016temporal,tarvainen2017mean} have  shown to work well on many benchmarks.

In a nutshell, consistency training methods simply regularize model predictions to be invariant to small noise applied to either input examples~\cite{miyato2018virtual, sajjadi2016regularization, clark2018semi} or hidden states~\cite{bachman2014learning, laine2016temporal}. 
This framework makes sense intuitively because a good model should be robust to any small change in an input example or hidden states.
Under this framework, different methods in this category differ mostly in how and where the noise injection is applied.  Typical noise injection methods are additive Gaussian noise, dropout noise or adversarial noise.

In this work, we investigate the role of noise injection in consistency training 
and observe that advanced data augmentation methods, specifically those work best in supervised learning \cite{simard1998transformation,krizhevsky2012imagenet,cubuk2018autoaugment,yu2018qanet}, also perform well in semi-supervised learning.
There is indeed a strong correlation between the performance of data augmentation operations in supervised learning and their performance in consistency training. We, hence, propose to substitute the traditional noise injection methods with high quality data augmentation methods in order to improve consistency training.
To emphasize the use of better data augmentation in consistency training, we name our method Unsupervised Data Augmentation or UDA.

We evaluate \name on a wide variety of language and vision tasks. 
On six text classification tasks, our method achieves significant improvements over state-of-the-art models. Notably, on IMDb, \name with 20 labeled examples outperforms the state-of-the-art model trained on 1250x more labeled data. On standard semi-supervised learning benchmarks CIFAR-10 and SVHN, 
UDA outperforms all existing semi-supervised learning methods by significant margins and achieves an error rate of 5.43 and 2.72 with 250 labeled examples respectively. 
Finally, we also find \name to be beneficial when there is a large amount of supervised data. 
For instance, on ImageNet, \name leads to improvements of top-1 accuracy from $58.84$ to $68.78$ with $10\%$ of the labeled set and from $78.43$ to $79.05$ when we use the full labeled set and an external dataset with $1.3$M unlabeled examples.

Our key contributions and findings can be summarized as follows:
\begin{itemize}[leftmargin=*,itemsep=0em,topsep=0em]
    \item First, we show that state-of-the-art data augmentations found in supervised learning can also serve as a superior source of noise under the consistency enforcing semi-supervised framework. {\it See results in Table~\ref{tab:cifar10_sup_vs_unsup} and Table~\ref{tab:yelp_5_sup_vs_unsup}.}
    \item Second, we show that UDA can match and even outperform purely supervised learning that uses orders of magnitude more labeled data. {\it See results in Table~\ref{tab:text_results} and Figure~\ref{fig:cifar_svhn_vary_sup}.}
    
    {\it State-of-the-art results for both vision and language tasks are reported in Table~\ref{tab:published_results} and \ref{tab:text_results}. The effectiveness of UDA across different training data sizes are highlighted in Figure~\ref{fig:cifar_svhn_vary_sup} and \ref{fig:text_vary_sup}.}
    \item Third, we show that UDA combines well with transfer learning, e.g., when fine-tuning from BERT ({\it see Table~\ref{tab:text_results}}), and is effective at high-data regime, e.g. on ImageNet ({\it see Table~\ref{tab:imagenet}}).
    \item Lastly, we also provide a theoretical analysis of how \name improves the classification performance and the corresponding role of the state-of-the-art augmentation in Section \ref{sec:theory}.
\end{itemize}

\vspace{-0.5em}
\section{Unsupervised Data Augmentation (UDA)}
\label{sec:method}
\vspace{-0.5em}
In this section, we first formulate our task and then present the key method and insights behind \name.
Throughout this paper, we focus on classification problems and will use $x$ to denote the input and $y^*$ to denote its ground-truth prediction target. We are interested in learning a model $p_{\theta}(y \mid x)$ to predict $y^*$ based on the input $x$, where $\theta$ denotes the model parameters.
Finally, we will use $p_L(x)$ and $p_U(x)$ to denote the distributions of labeled and unlabeled examples respectively and use $f^*$ to denote the perfect classifier that we hope to learn. 

\vspace{-0.5em}
\subsection{Background: Supervised Data Augmentation}
\label{sec:sda}
\vspace{-0.5em}

Data augmentation aims at creating novel and realistic-looking training data by applying a transformation to an example, without changing its label.
Formally, let $q(\hat{x} \mid x)$ be the augmentation transformation from which one can draw augmented examples $\hat{x}$ based on an original example $x$.
For an augmentation transformation to be valid, it is required that any example $\hat{x} \sim q(\hat{x} \mid x)$ drawn from the distribution shares the same ground-truth label as $x$.
Given a valid augmentation transformation, we can simply minimize the negative log-likelihood on augmented examples. 

Supervised data augmentation can be equivalently seen as constructing an augmented labeled set from the original supervised set and then training the model on the augmented set. Therefore, the augmented set needs to provide additional inductive biases to be more effective. How to design the augmentation transformation has, thus, become critical. 

In recent years, there have been significant advancements on the design of data augmentations for NLP~\cite{yu2018qanet}, vision~\cite{krizhevsky2012imagenet,cubuk2018autoaugment} and speech~\cite{hannun2014deep,park2019specaugment} in supervised settings.
Despite the promising results, data augmentation is mostly regarded as the ``cherry on the cake'' which provides a steady but limited performance boost because these augmentations has so far only been applied to a set of labeled examples which is usually of a small size. 
Motivated by this limitation, via the consistency training framework, we extend the advancement in supervised data augmentation to semi-supervised learning where abundant unlabeled data is available. 

\vspace{-0.3em}
\subsection{Unsupervised Data Augmentation}
\label{sec:uda}
\vspace{-0.3em}

As discussed in the introduction, a recent line of work in semi-supervised learning has been utilizing unlabeled examples to enforce smoothness of the model.
The general form of these works can be summarized as follows:
\begin{itemize}[leftmargin=*,itemsep=0em,topsep=0em]
\item Given an input $x$, compute the output distribution $p_{\theta}(y \mid x)$ given $x$ and a noised version $p_{\theta}(y \mid x, \epsilon)$ by injecting a small noise $\epsilon$. 
The noise can be applied to $x$ or hidden states.
\item Minimize a divergence metric between the two distributions $\divg{p_{\theta}(y \mid x)}{p_{\theta}(y \mid x, \epsilon)}$.
\end{itemize}
This procedure enforces the model to be insensitive to the noise $\epsilon$ and hence smoother with respect to changes in the input (or hidden) space. 
From another perspective, minimizing the consistency loss gradually propagates label information from labeled examples to unlabeled ones.

In this work, we are interested in a particular setting where the noise is injected to the input $x$, i.e., $\hat{x} = q(x, \epsilon)$, as considered by prior works~\cite{sajjadi2016regularization, laine2016temporal, miyato2018virtual}.
But different from existing work, we focus on the unattended question of how the form or ``quality'' of the noising operation $q$ can influence the performance of this consistency training framework.
Specifically, to enforce consistency, prior methods generally employ simple noise injection methods such as adding Gaussian noise, simple input augmentations to noise unlabeled examples. 
In contrast, we hypothesize that stronger data augmentations in supervised learning can also lead to superior performance when used to noise unlabeled examples in the semi-supervised consistency training framework, since it has been shown that more advanced data augmentations that are more diverse and natural can lead to significant performance gain in the supervised setting.

Following this idea, we propose to use a rich set of state-of-the-art data augmentations verified in various supervised settings to inject noise and optimize the same consistency training objective on unlabeled examples. When jointly trained with labeled examples, we utilize a weighting factor $\lambda$ to balance the supervised cross entropy and the unsupervised consistency training loss, which is illustrated in Figure \ref{fig:illustration}. Formally, the full objective can be written as follows:
\begin{equation}
\label{eq:uda}
\min_{\theta}\; \mathcal{J}(\theta) = \E_{x_1\sim p_L(x)}\left[-\log p_{\theta} (f^*(x_1) \mid x_1) \right] + \lambda \E_{x_2 \sim p_U(x)} \E_{\hat{x} \sim q(\hat{x} \mid x_2)} \left[\mathrm{CE}\left(p_{\tilde{\theta}} (y \mid x_2) \| p_{\theta} (y \mid \hat{x}) \right) \right]
\end{equation}
where CE denotes cross entropy, $q(\hat{x} \mid x)$ is a data augmentation transformation and $\tilde{\theta}$ is a \textit{fixed} copy of the current parameters $\theta$ indicating that the gradient is not propagated through $\tilde{\theta}$, as suggested by VAT~\cite{miyato2018virtual}. 
We set $\lambda$ to $1$ for most of our experiments. In practice, in each iteration, we compute the supervised loss on a mini-batch of labeled examples and compute the consistency loss on a mini-batch of unlabeled data. The two losses are then summed for the final loss. We use a larger batch size for the consistency loss. 

In the vision domain, simple augmentations including cropping and flipping are applied to labeled examples. To minimize the discrepancy between supervised training and prediction on unlabeled examples, we apply the same simple augmentations to unlabeled examples for computing $p_{\tilde{\theta}} (y \mid x)$.

\begin{figure}
\vspace{-3em}
    \centering
    \includegraphics[width=0.85\textwidth]{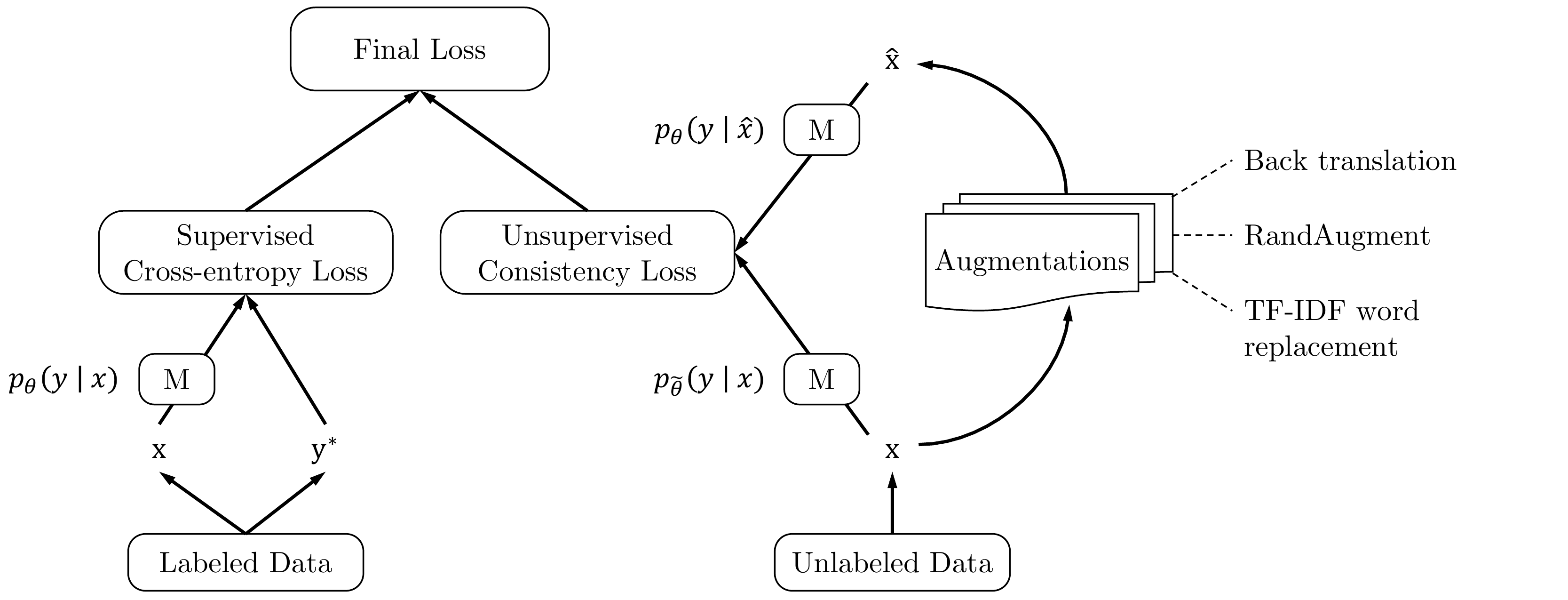}
    \caption{Training objective for \name, where M is a model that predicts a distribution of $y$ given $x$.}
    \label{fig:illustration}
\vspace{-1em}
\end{figure}

\textbf{Discussion.}
Before detailing the augmentation operations used in this work, we first provide some intuitions on how more advanced data augmentations can provide extra advantages over simple ones used in earlier works from three aspects: 
\begin{itemize}[leftmargin=*,itemsep=0em,topsep=0em]
\item \textbf{Valid noise}: Advanced data augmentation methods that achieve great performance in supervised learning usually generate realistic augmented examples that share the same ground-truth labels with the original example.
Thus, it is safe to encourage the consistency between predictions on the original unlabeled example and the augmented unlabeled examples.

\item \textbf{Diverse noise}: Advanced data augmentation can generate a diverse set of examples since it can make large modifications to the input example without changing its label, while simple Gaussian noise only make local changes. Encouraging consistency on a diverse set of augmented examples can significantly improve the sample efficiency. 

\item \textbf{Targeted inductive biases}: Different tasks require different inductive biases. Data augmentation operations that work well in supervised training essentially provides the missing inductive biases.
\end{itemize}

\vspace{-0.3em}
\subsection{Augmentation Strategies for Different Tasks} 
\vspace{-0.3em}
\label{sec:data_aug_for_task}
We now detail the augmentation methods, tailored for different tasks, that we use in this work.

\textbf{RandAugment for Image Classification.}
We use a data augmentation method called RandAugment~\cite{cubuk2019RandAugment}, which is inspired by AutoAugment~\cite{cubuk2018autoaugment}. AutoAugment uses a search method to combine all image processing transformations in the Python Image Library (PIL) to find a good augmentation strategy.
In RandAugment, we do not use search, but instead uniformly sample from the same set of augmentation transformations in PIL. In other words, RandAugment is simpler and requires no labeled data as there is no need to search for optimal policies.

\textbf{Back-translation for Text Classification.} 
When used as an augmentation method, back-translation~\cite{sennrich2015improving, edunov2018understanding} refers to the procedure of translating an existing example $x$ in language $A$ into another language $B$ and then translating it back into $A$ to obtain an augmented example $\hat{x}$.
As observed by \cite{yu2018qanet}, back-translation can generate diverse paraphrases while preserving the semantics of the original sentences, leading to significant performance improvements in question answering. 
In our case, we use back-translation to paraphrase the training data of our text classification tasks.\footnote{We also note that while translation uses a labeled dataset, the translation task itself is quite distinctive from a text classification task and does not make use of any text classification label. In addition, back-translation is a general data augmentation method that can be applied to many tasks with the same model checkpoints.}

We find that the diversity of the paraphrases is important. Hence, we employ random sampling with a tunable temperature instead of beam search for the generation. As shown in Figure \ref{fig:augmentation}, the paraphrases generated by back-translation sentence are diverse and have similar semantic meanings.
More specifically, we use WMT'14 English-French translation models (in both directions) to perform back-translation on each sentence.
To facilitate future research, we have open-sourced our back-translation system together with the translation checkpoints.

\begin{figure}[h]
\vspace{-0.7em}
    \centering
    \includegraphics[width=0.85\textwidth]{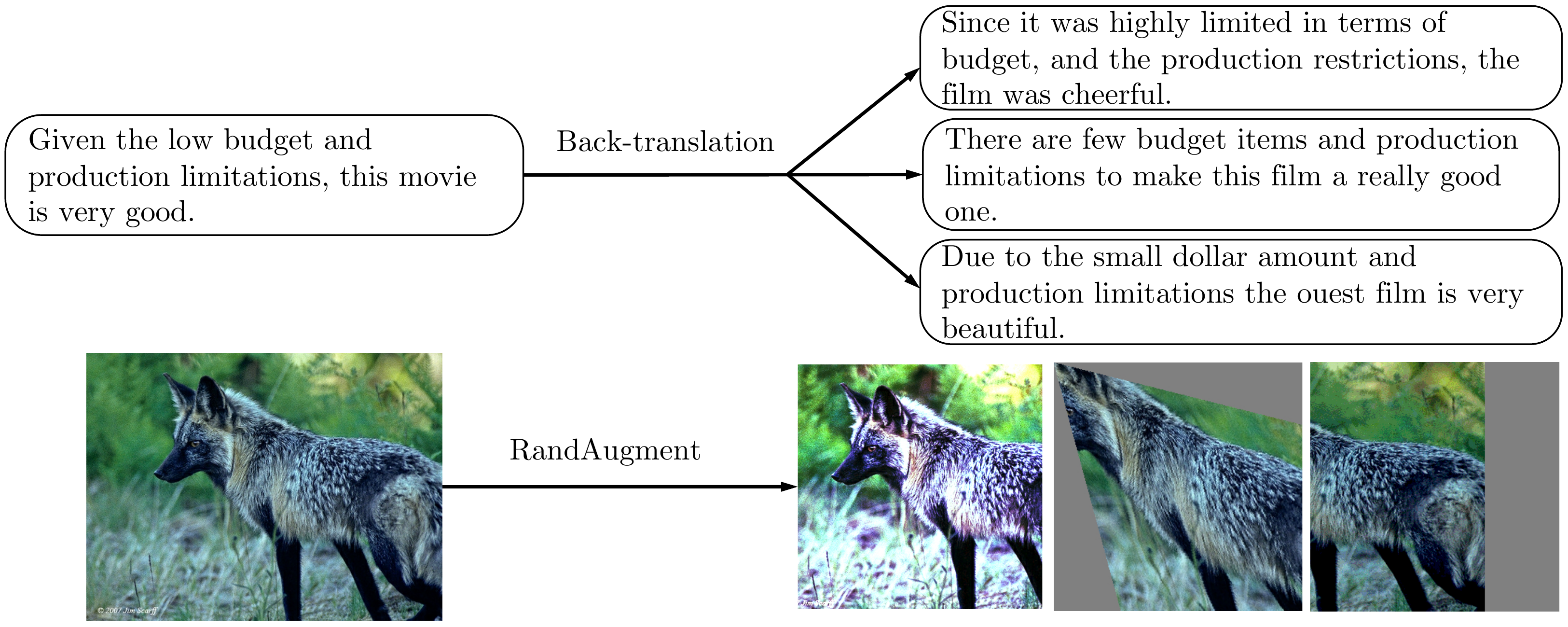}
    \caption{Augmented examples using back-translation and RandAugment. }
    \label{fig:augmentation}
\vspace{-0.7em}

\end{figure}

\textbf{Word replacing with TF-IDF for Text Classification.} While back-translation is good at maintaining the global semantics of a sentence, there is little control over which words will be retained. This requirement is important for topic classification tasks, such as DBPedia, in which some keywords are more informative than other words in determining the topic. We, therefore, propose an augmentation method that replaces uninformative words with low TF-IDF scores while keeping those with high TF-IDF values. We refer readers to Appendix \ref{sec:apdx_augmentations} for a detailed description.

\vspace{-0.3em}
\subsection{Additional Training Techniques}
\label{sec:apdx_training_technique}
\vspace{-0.3em}

In this section, we present additional techniques targeting at some commonly encountered problems.

\textbf{Confidence-based masking.} We find it to be helpful to mask out examples that the current model is not confident about. Specifically, in each minibatch, the consistency loss term is computed only on examples whose highest probability among classification categories is greater than a threshold $\beta$. We set the threshold $\beta$ to a high value. Specifically, $\beta$ is set to 0.8 for CIFAR-10 and SVHN and 0.5 for ImageNet. 

\textbf{Sharpening Predictions.}
\label{sec:sharpending_predictions}
Since regularizing the predictions to have low entropy has been shown to be beneficial~\cite{grandvalet2005semi,miyato2018virtual}, we sharpen predictions when computing the target distribution on unlabeled examples by using a low Softmax temperature $\tau$. 
When combined with confidence-based masking, the loss on unlabeled examples $\E_{x \sim p_U(x)} \E_{\hat{x} \sim q(\hat{x} \mid x)} \left[\mathrm{CE}\left(p_{\tilde{\theta}} (y \mid x) \| p_{\theta} (y \mid \hat{x}) \right) \right]$ on a minibatch $B$ is computed as:
\begin{equation*}
    \frac{1}{|B|} \sum_{x\in B} I(\max_{y'}p_{\tilde{\theta}}(y' \mid x)>\beta)\mathrm{CE}\left(p^{(sharp)}_{\tilde{\theta}} (y \mid x) \| p_{\theta} (y \mid \hat{x}) \right) 
\end{equation*}
\begin{equation*}
    p^{(sharp)}_{\tilde{\theta}} (y \mid x) = \frac{\exp(z_{y} / \tau)}{\sum_{y'} \exp(z_{y'} / \tau)}
\end{equation*}
where $I(\cdot)$ is the indicator function, $z_y$ is the logit of label $y$ for example $x$. We set $\tau$ to 0.4 for CIFAR-10, SVHN and ImageNet.

\textbf{Domain-relevance Data Filtering.} 
\label{sec:domain_relevance}
Ideally, we would like to make use of out-of-domain unlabeled data since it is usually much easier to collect, but the class distributions of out-of-domain data are mismatched with those of in-domain data, which can result in performance loss if directly used~\cite{oliver2018realistic}.
To obtain data relevant to the domain for the task at hand, we adopt a common technique for detecting out-of-domain data. 
We use our baseline model trained on the in-domain data to infer the labels of data in a large out-of-domain dataset and pick out examples that the model is most confident about. 
Specifically, for each category, we sort all examples based on the classified probabilities of being in that category and select the examples with the highest probabilities.

\vspace{-0.5em}
\section{Theoretical Analysis}
\vspace{-0.5em}
\label{sec:theory}
In this section, we theoretically analyze why \name can improve the performance of a model and the required number of labeled examples to achieve a certain error rate.
Following previous sections, we will use $f^*$ to denote the perfect classifier that we hope to learn, use $p_U$ to denote the marginal distribution of the unlabeled data and use $q(\hat{x}\mid x)$ to denote the augmentation distribution.

To make the analysis tractable, we make the following simplistic assumptions about the data augmentation transformation: 
\begin{itemize}[leftmargin=*,itemsep=0em,topsep=0em]
    \item \textbf{In-domain} augmentation: data examples generated by data augmentation have non-zero probability under $p_U$, i.e., $p_U(\hat{x}) > 0$ for $\hat{x} \sim q(\hat{x} \mid x), x \sim p_U(x)$.
    \item \textbf{Label-preserving} augmentation: data augmentation preserves the label of the original example, i.e., $f^*(x)=f^*(\hat{x})$ for $\hat{x} \sim q(\hat{x} \mid x), x \sim p_U(x)$.
    \item \textbf{Reversible} augmentation: the data augmentation operation can be reversed, i.e., if $q(\hat{x} \mid x) >0$ then $q(x \mid \hat{x})>0$ .
\end{itemize}

As the first step, we hope to provide an intuitive sketch of our formal analysis.
Let us define a graph $G_{p_U}$ where each node corresponds to a data sample $x \in X$ and an edge $(\hat{x}, x)$ exists in the graph \textit{if and only if} $q(\hat{x} \mid x) > 0$. 
Due to the label-preserving assumption, it is easy to see that examples with different labels must reside on different components (disconnected sub-graphs) of the graph $G_{p_U}$.
Hence, for an $N$-category classification problems, the graph has  $N$ components (sub-graphs) when all examples within each category can be traversed by the augmentation operation.
Otherwise, the graph will have more than $N$ components.

Given this construction, notice that for each component $C_i$ of the graph, as long as there is a single labeled example in the component, i.e. $(x^*, y^*) \in C_i$, one can propagate the label of the node to the rest of the nodes in $C_i$ by traversing  $C_i$ via the augmentation operation $q(\hat{x} \mid x)$.
More importantly, if one only performs \textit{supervised data augmentation}, one can only propagate the label information to the directly connected neighbors of the labeled node.
In contrast, performing \textit{unsupervised data augmentation} ensures the traversal of the entire sub-graph $C_i$.
This provides the first high-level intuition how \name could  help.

Taking one step further, in order to find a perfect classifier via such label propagation, it requires that there exists at least one labeled example in each component.
In other words, the number of components lower bounds the minimum amount of labeled examples needed to learn a perfect classifier.
Importantly, number of components is actually decided by the quality of the augmentation operation: an ideal augmentation should be able to reach all other examples of the same category given a starting instance.
This well matches our discussion of the benefits of state-of-the-art data augmentation methods in generating more diverse examples.
Effectively, the augmentation diversity leads to more neighbors for each node, and hence reduces the number of components in a graph. 

Since supervised data augmentation only propagates the label information to the directly connected neighbors of the labeled nodes.
Advanced data augmentation that has a high accuracy must lead to a graph where each node has more neighbors. Effectively, such a graph has more edges and better connectivity. Hence, it is also more likely that this graph will have a smaller number of components. 
To further illustrate this intuition, in Figure \ref{fig:alg_compare}, we provide a comparison between different algorithms.

\begin{figure}[ht]
\begin{subfigure}{.19\textwidth}
  \centering
  \includegraphics[width=0.8\linewidth]{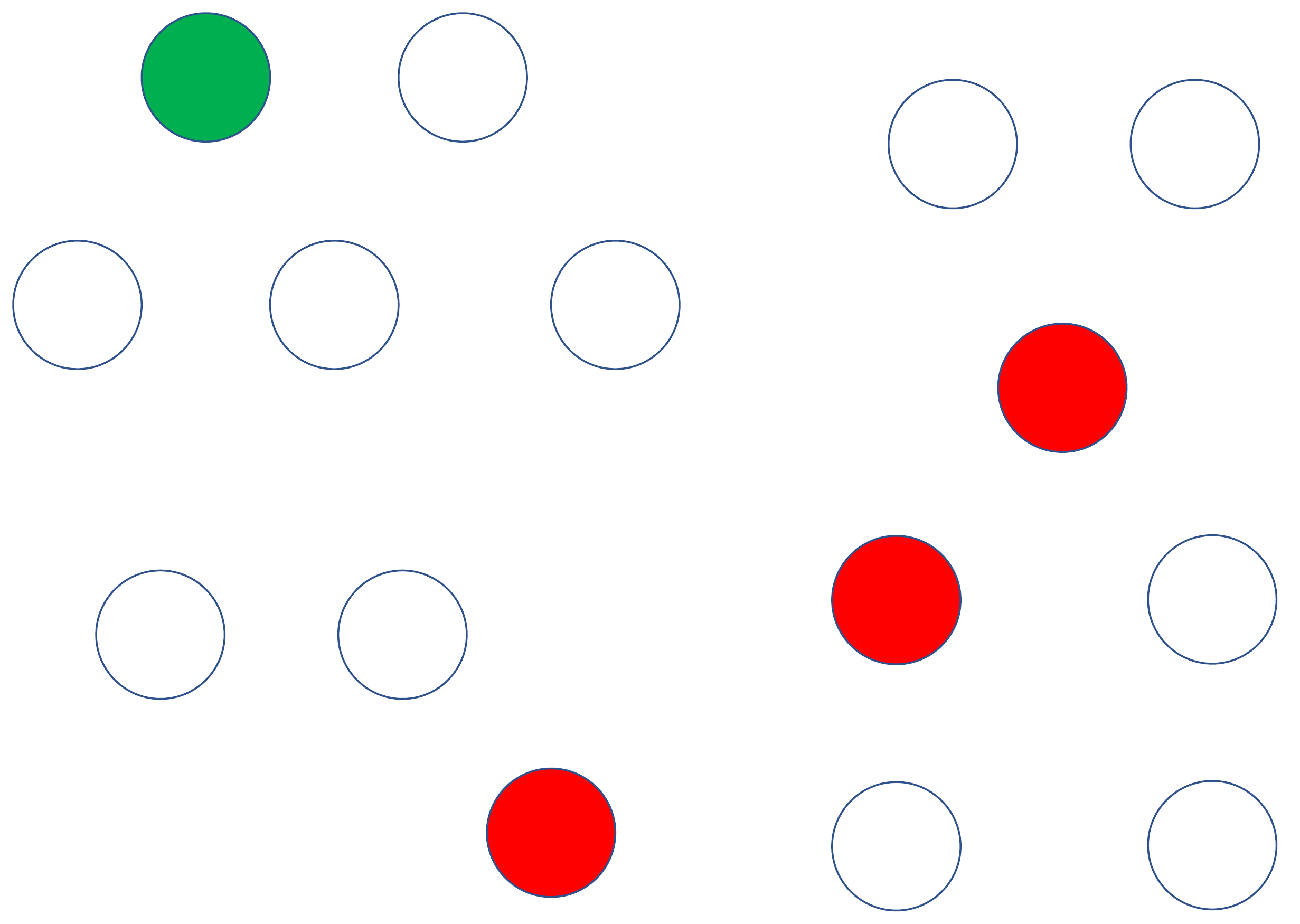}  
  \caption{Supervised learning. \\(4/15)}
  \label{fig:sub-first}
\end{subfigure}
\hfill
\begin{subfigure}{.19\textwidth}
  \centering
  \includegraphics[width=0.8\linewidth]{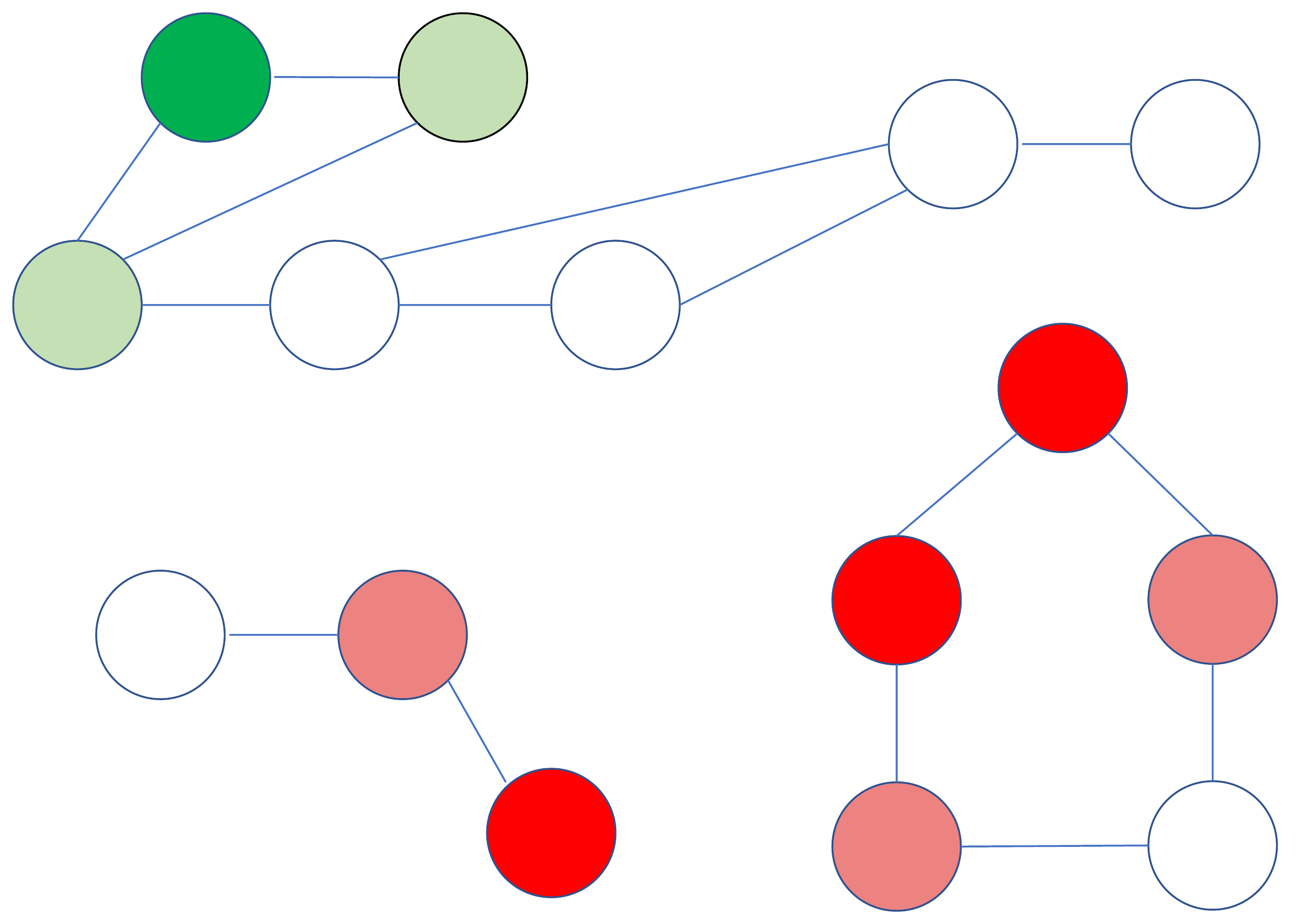}  
  \caption{Advanced supervised augmentation.\\ (9/15)}
  \label{fig:sub-second}
\end{subfigure}
\hfill
\begin{subfigure}{.19\textwidth}
  \centering
  \includegraphics[width=0.8\linewidth]{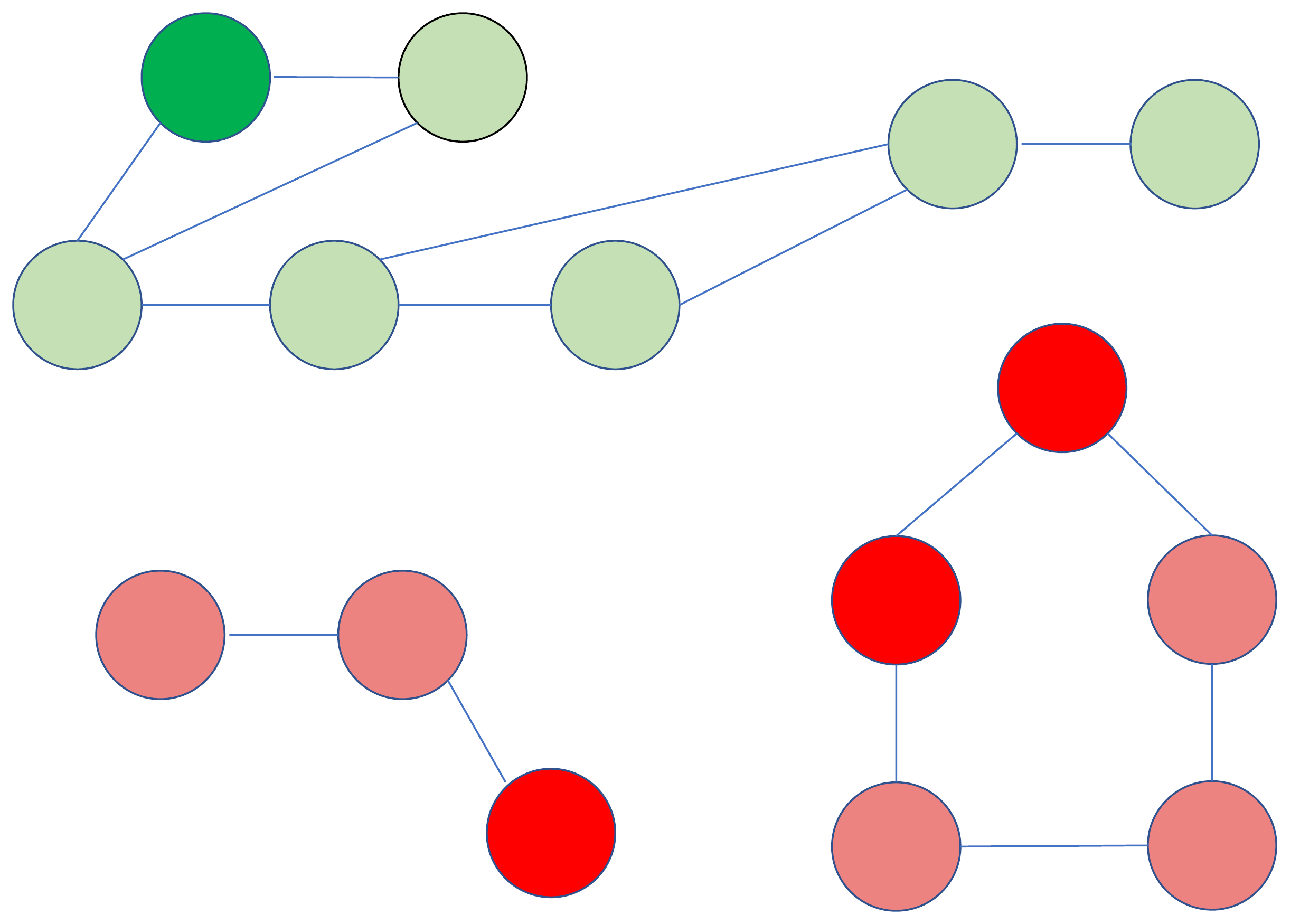}  
  \caption{UDA with advanced augmentation. (15/15)}
  \label{fig:sub-second}
\end{subfigure}
\hfill
\begin{subfigure}{.19\textwidth}
  \centering
  \includegraphics[width=0.8\linewidth]{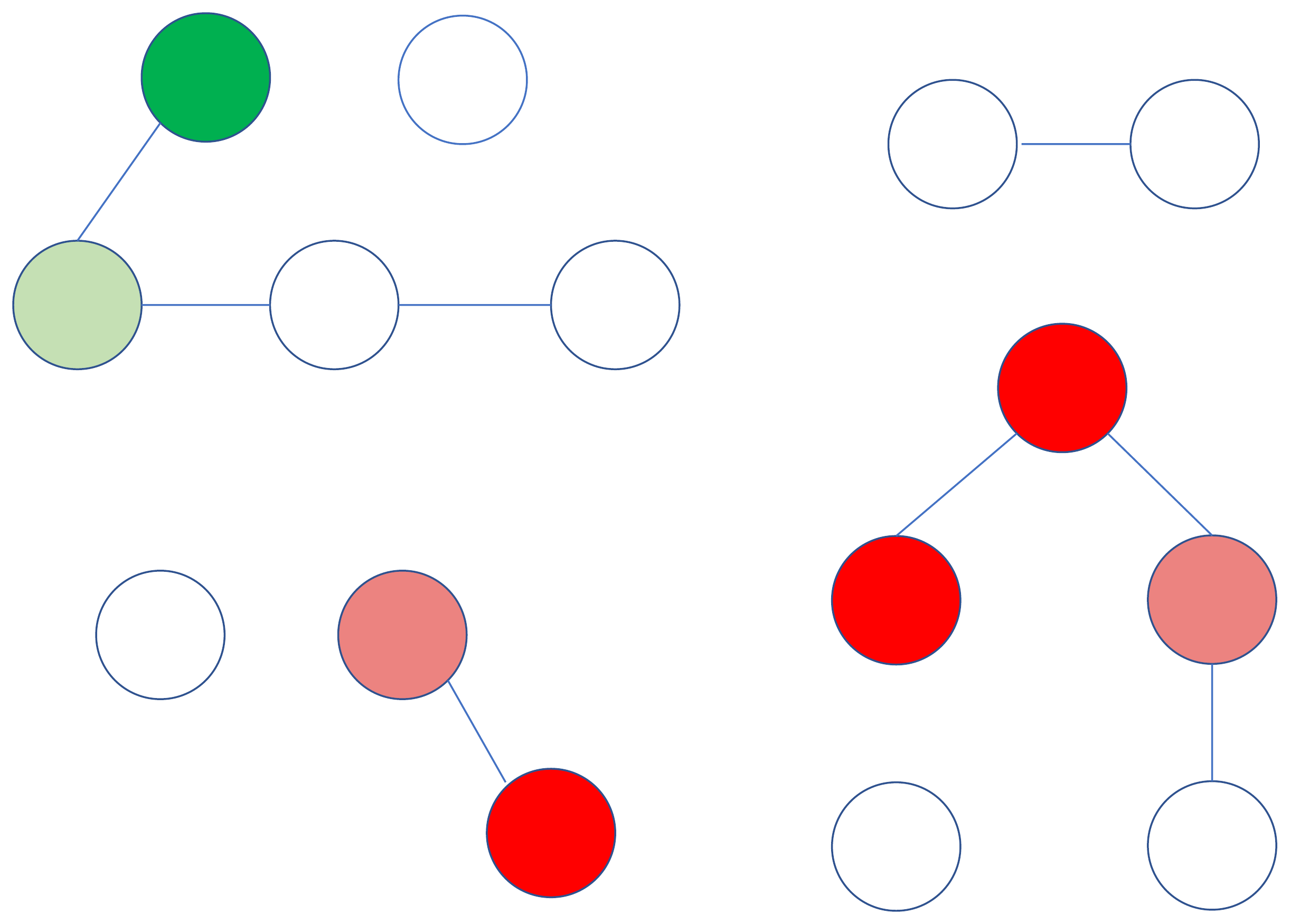}  
  \caption{Simple supervised augmentation.\\ (7/15)}
  \label{fig:sub-second}
\end{subfigure}
\hfill
\begin{subfigure}{.19\textwidth}
  \centering
  \includegraphics[width=0.8\linewidth]{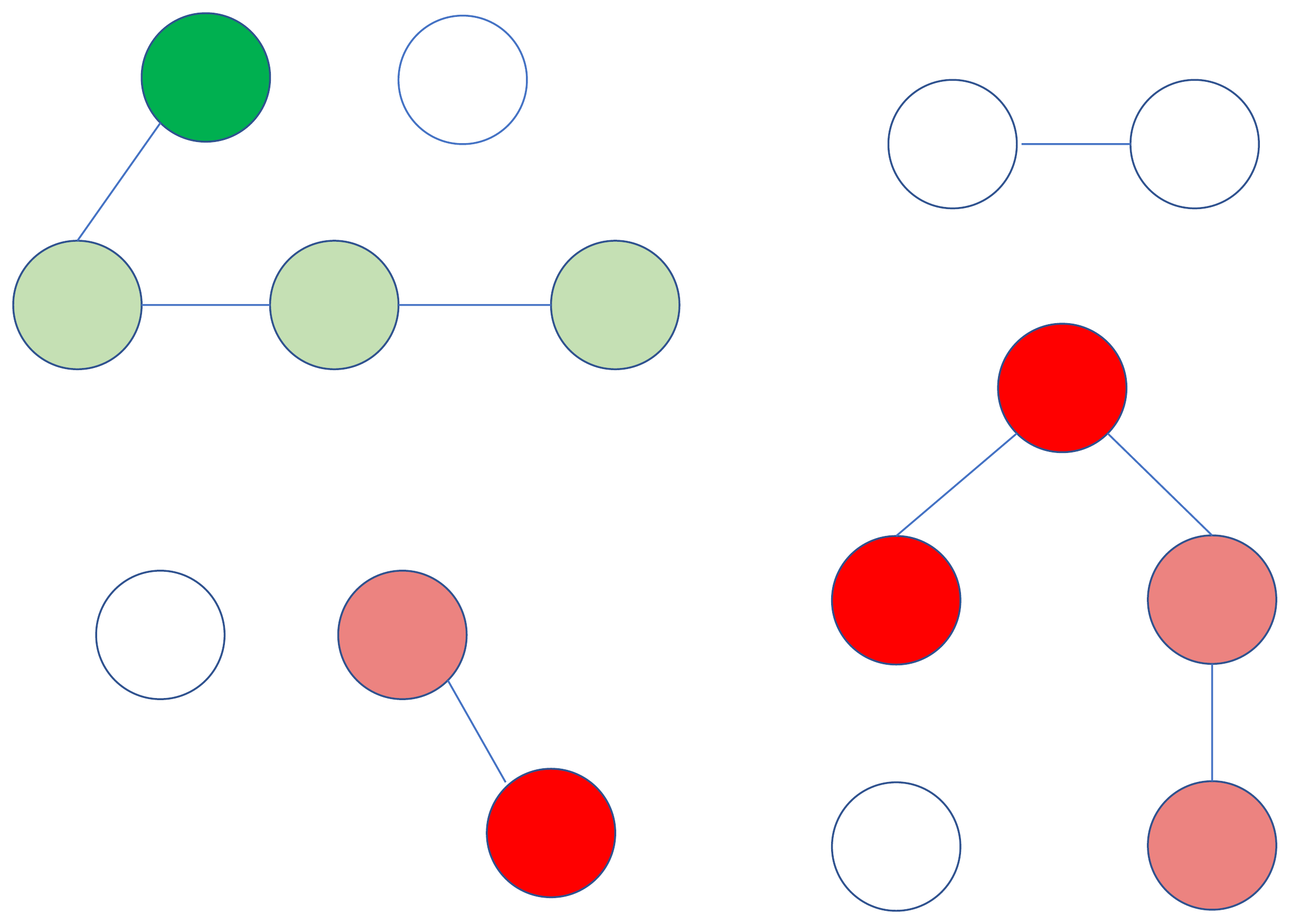}  
  \caption{UDA with simple augmentation.\\ (10/15)}
  \label{fig:sub-second}
\end{subfigure}
\caption{Prediction results of different settings, where green and red nodes are labeled nodes, white nodes are unlabeled nodes whose labels cannot be determined and light green nodes and light red nodes are unlabeled nodes whose labels can be correctly determined. The accuracy of different settings are shown in $(\cdot)$.}
\label{fig:alg_compare}
\vspace{-0.7em}
\end{figure}

With the intuition described, we state our formal results.
Without loss of generality, assume there are $k$ components in the graph.
For each component $C_i (i = 1, \dots, k)$, let $P_i$ be the total probability mass that an observed labeled example fall into the $i$-th component, i.e., $P_i = \sum_{x \in C_i} p_L(x)$.
The following theorem characterizes the relationship between UDA error rate and the amount of labeled examples.

\begin{theorem}\label{thm:uda_label_requirement}
Under UDA, let $Pr(\mathcal{A})$ denote the probability that the algorithm \textit{cannot} infer the label of a new test example given $m$ labeled examples from $P_{L}$. $Pr(\mathcal{A})$ is given by
\[
    Pr(\mathcal{A})= \sum_{i} P_i(1 - P_i)^m.
\]
In addition,  $O(k / \epsilon)$ labeled examples can guarantee an error rate of $O(\epsilon)$, i.e.,
\[
    m = O(k/\epsilon) \implies Pr(\mathcal{A}) = O(\epsilon).
\]
\end{theorem}
\begin{proof}
Please see Appendix. \ref{sec:proof} for details.
\vspace{-1em}
\end{proof}
From the theorem, we can see the number of components, i.e. $k$, directly governs the amount of labeled data required to reach a desired performance.
As we have discussed above, the number of components effectively relies on the quality of an augmentation function, where better augmentation functions result in fewer components. 
This echoes our discussion of the benefits of  state-of-the-art data augmentation operations in generating more diverse examples.
Hence, with state-of-the-art augmentation operations, \name is able to achieve good performance using fewer labeled examples. 

\vspace{-0.3em}
\section{Experiments}
\vspace{-0.3em}

In this section, we evaluate \name on a variety of language and vision tasks.
For language, we rely on six text classification benchmark datasets, including IMDb, Yelp-2, Yelp-5, Amazon-2 and Amazon-5 sentiment classification and DBPedia topic classification~\cite{maas2011learning, zhang2015character}. 
For vision, we employ two smaller datasets CIFAR-10~\cite{krizhevsky2009learning}, SVHN~\cite{netzer2011reading}, which are often used to compare semi-supervised algorithms, as well as ImageNet~\cite{deng2009imagenet} of a larger scale to test the scalability of \name.
For ablation studies and experiment details, we refer readers to Appendix  \ref{apdx:exp} and Appendix \ref{sec:exp_details}.

\vspace{-0.3em}
\subsection{Correlation between Supervised and Semi-supervised Performances}
\vspace{-0.3em}

As the first step, we try to verify the fundamental idea of \name, i.e., there is a positive correlation of data augmentation's effectiveness in supervised learning and semi-supervised learning.
Based on Yelp-5 (a language task) and CIFAR-10 (a vision task), we compare the performance of different data augmentation methods in either fully supervised or semi-supervised settings.
For Yelp-5, apart from back-translation, we include a simpler method Switchout~\cite{wang2018switchout} which replaces a token with a  random token uniformly sampled from the vocabulary.
For CIFAR-10, we compare RandAugment with two simpler methods: (1) cropping \& flipping augmentation and (2) Cutout.

\begin{table}[h!]
\vspace{-0.3em}
\mbox{}\hfill
\begin{minipage}[t]{0.48\textwidth}
\footnotesize
\centering
	
\begin{tabular}{l|ccc}
		\toprule
		{\bf Augmentation} & Sup & Semi-Sup   \\
		(\# Sup examples) & (50k) & (4k) \\
		\midrule 
		Crop \& flip & 5.36 & 10.94 \\
		Cutout & 4.42 & 5.43 \\
		RandAugment & \textbf{4.23} & \textbf{4.32}  \\
		\bottomrule
	\end{tabular}
		\caption{Error rates on CIFAR-10.}
	\label{tab:cifar10_sup_vs_unsup}
	
\end{minipage}
~\hfill
\begin{minipage}[t]{0.48\textwidth}
	\centering
	\footnotesize
	
\begin{tabular}{l|ccc}
		\toprule
		{\bf Augmentation} & Sup & Semi-sup  \\
        (\# Sup examples) & (650k) & (2.5k) \\
		\midrule 
		\xmark  & 38.36 & 50.80 \\ 
		Switchout & 37.24  & 43.38  \\
		Back-translation  &  \textbf{36.71} & \textbf{41.35}  \\
		\bottomrule
	\end{tabular}
		\caption{Error rate on Yelp-5.}
	\label{tab:yelp_5_sup_vs_unsup}
	
\end{minipage}
\vspace{-1.5em}
\end{table}

Based on this setting, Table \ref{tab:cifar10_sup_vs_unsup} and Table \ref{tab:yelp_5_sup_vs_unsup}  exhibit a strong correlation of an augmentation's effectiveness between supervised and semi-supervised settings. This validates our idea of stronger data augmentations found in supervised learning can always lead to more gains when applied to the semi-supervised learning settings.

\vspace{-0.3em}

\subsection{Algorithm Comparison on Vision Semi-supervised Learning Benchmarks}
\label{sec:cifar_svhn_exp}
\vspace{-0.3em}

With the correlation established above, the next question we ask is how well \name performs compared to existing semi-supervised learning algorithms.
To answer the question, we focus on the most commonly used semi-supervised learning benchmarks CIFAR-10 and SVHN.

\textbf{Vary the size of labeled data.} Firstly, we follow the settings in \cite{oliver2018realistic} and employ Wide-ResNet-28-2~\cite{zagoruyko2016wide,he2016deep} as the backbone model and evaluate \name with varied supervised data sizes.
Specifically, we compare \name with two highly competitive baselines: (1) Virtual adversarial training (VAT)~\cite{miyato2018virtual}, an algorithm that generates adversarial Gaussian noise on input, and (2) MixMatch~\cite{berthelot2019mixmatch},  a parallel work that combines previous advancements in semi-supervised learning. The comparison is shown in Figure \ref{fig:cifar_svhn_vary_sup} with two key observations.

\begin{itemize}[leftmargin=*,itemsep=0em,topsep=0em]
\item First, \name consistently outperforms the two baselines given different sizes of labeled data.
\item Moreover, the performance difference between \name and VAT shows the superiority of data augmentation based noise. 
The difference of \name and VAT is essentially the noise process. While the noise produced by VAT often contain high-frequency artifacts that do not exist in real images, data augmentation mostly generates diverse and realistic images.
\end{itemize}
\begin{figure}[h!]
\vspace{-1em}
    \centering
    \begin{subfigure}[t]{0.45\textwidth}
        \centering
    \includegraphics[width=\textwidth]{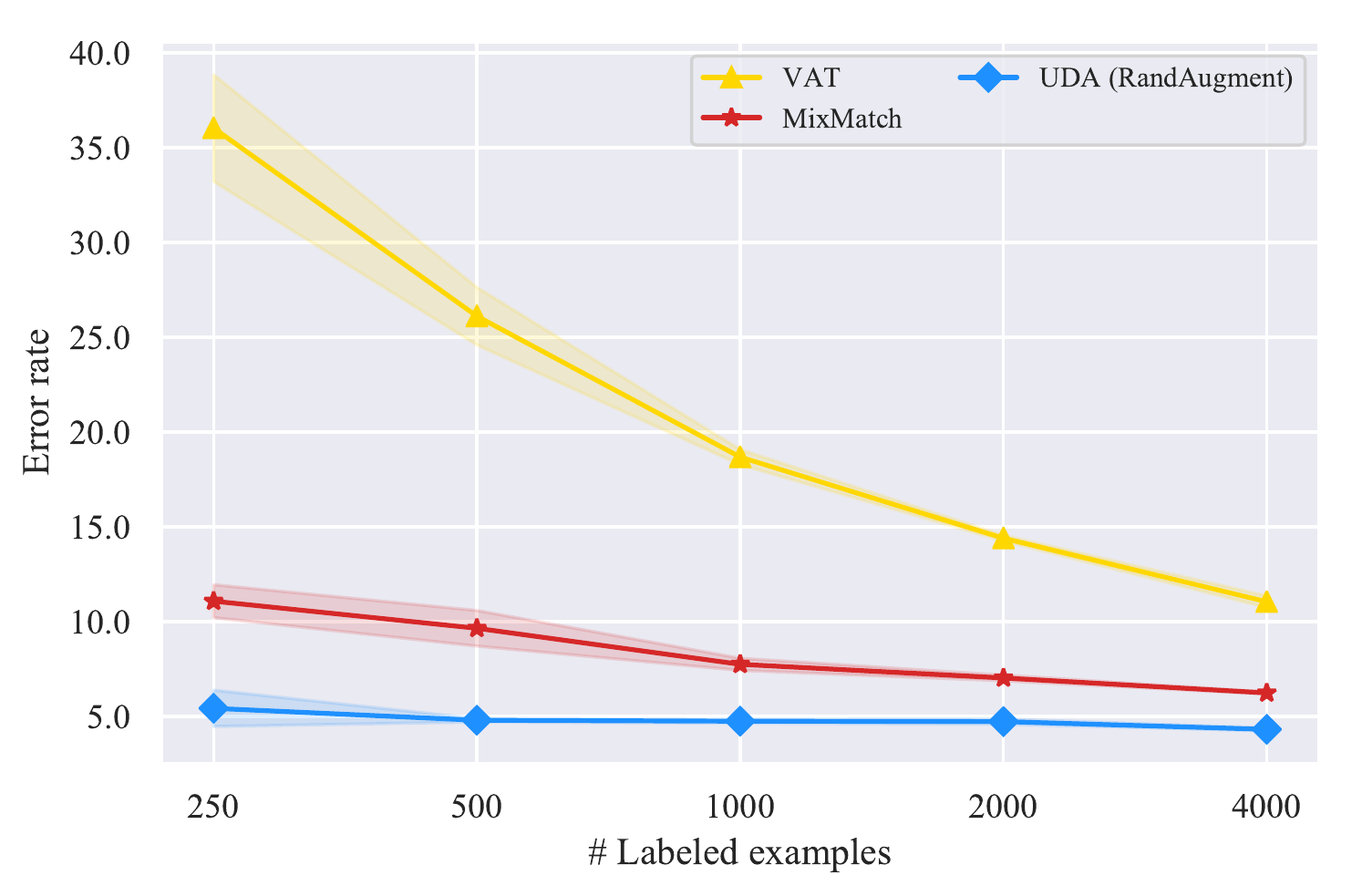}
    \caption{CIFAR-10}
    \label{fig:cifar_vary_sup}
    \end{subfigure}%
    ~ 
    \begin{subfigure}[t]{0.45\textwidth}
        \centering
	\includegraphics[width=\textwidth]{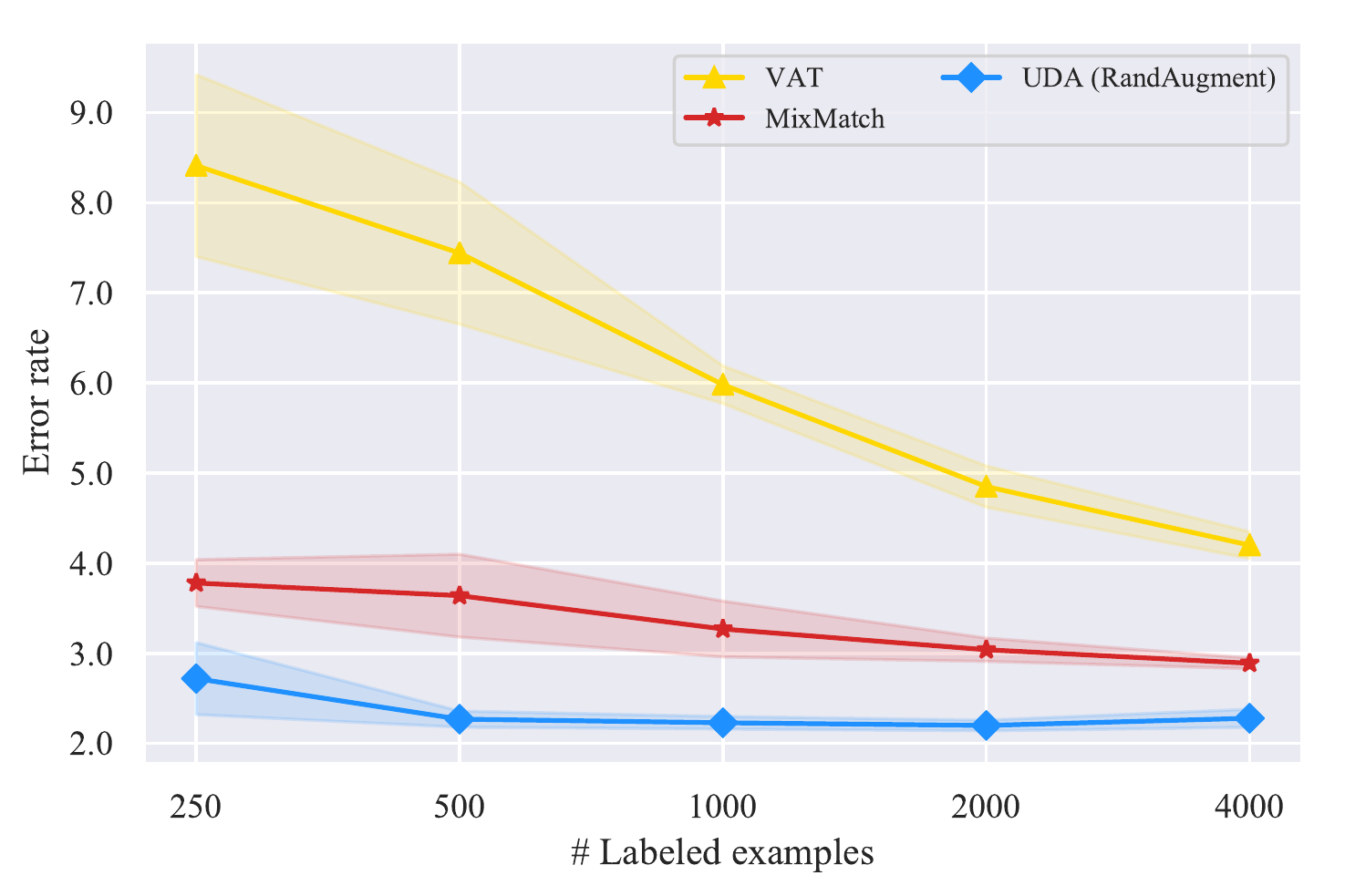}
	\caption{SVHN}
	\label{fig:svhn_vary_sup}
    \end{subfigure}
\caption{Comparison with two semi-supervised learning methods on CIFAR-10 and SVHN with varied number of labeled examples.  }
\label{fig:cifar_svhn_vary_sup}
\vspace{-1em}
\end{figure}
 
\textbf{Vary model architecture.} Next, we directly compare \name with previously published results under different model architectures.
Following previous work, 4k and 1k labeled examples are used for CIFAR-10 and SVHN respectively.
As shown in Table \ref{tab:published_results}, given the same architecture, UDA outperforms all published results by significant margins and nearly matches the fully supervised performance, which uses 10x more labeled examples.
This shows the huge potential of state-of-the-art data augmentations under the consistency training framework in the vision domain.

\begin{table}[h!]
	\centering
	\footnotesize
    \begin{tabular}{l|cc|cc}
    \toprule
    	{\bf Method} & {\bf Model} & {\bf \# Param}  & CIFAR-10 (4k) & SVHN (1k) \\
    	\midrule 
    	$\Pi$-Model~\cite{laine2016temporal} & Conv-Large & 3.1M & 12.36 $\pm$ 0.31 & 4.82 $\pm$ 0.17 \\
    	Mean Teacher~\cite{tarvainen2017mean} & Conv-Large & 3.1M & 12.31 $\pm$ 0.28  &  3.95 $\pm$ 0.19 \\
    	VAT + EntMin~\cite{miyato2018virtual} & Conv-Large & 3.1M & 10.55 $\pm$ 0.05 &  3.86 $\pm$ 0.11 \\
    	SNTG~\cite{luo2018smooth} & Conv-Large & 3.1M & 10.93 $\pm$ 0.14 & 3.86 $\pm$ 0.27  \\
    	ICT~\cite{verma2019interpolation} & Conv-Large & 3.1M & 7.29 $\pm$ 0.02 &   3.89 $\pm$ 0.04 \\
    	Pseudo-Label~\cite{lee2013pseudo} & WRN-28-2 & 1.5M &  16.21 $\pm$ 0.11 & 7.62 $\pm$ 0.29 \\
    	LGA + VAT~\cite{jackson2019semi} & WRN-28-2 & 1.5M & 12.06 $\pm$ 0.19 &  6.58 $\pm$ 0.36 \\
        ICT~\cite{verma2019interpolation} & WRN-28-2 & 1.5M & 7.66 $\pm$ 0.17 &  3.53 $\pm$ 0.07 \\
        MixMatch~\cite{berthelot2019mixmatch} & WRN-28-2 & 1.5M & 6.24 $\pm$ 0.06 & 2.89 $\pm$ 0.06 \\
    	Mean Teacher~\cite{tarvainen2017mean} & Shake-Shake & 26M & 6.28 $\pm$ 0.15  &  - \\

        Fast-SWA~\cite{athiwaratkun2018there} & Shake-Shake & 26M & 5.0 & -\\
        MixMatch~\cite{berthelot2019mixmatch} & WRN & 26M & 4.95 $\pm$ 0.08 & -\\

        \midrule
    	 \name (RandAugment) & WRN-28-2 & 1.5M & 4.32 $\pm$ 0.08 &  \textbf{2.23 $\pm$ 0.07} \\
    	 \name (RandAugment) & Shake-Shake  & 26M & 3.7 & - \\
    	\name (RandAugment) & PyramidNet  & 26M & \textbf{2.7} & - \\
    	\bottomrule
    \end{tabular}
        \caption{Comparison between methods using different models where PyramidNet is used with ShakeDrop regularization. On CIFAR-10, with only 4,000 labeled examples, \name matches the performance of fully supervised Wide-ResNet-28-2 and PyramidNet+ShakeDrop, where they have an error rate of 5.4 and 2.7 respectively when trained on 50,000 examples without RandAugment.  On SVHN, \name also matches the performance of our fully supervised model trained on 73,257 examples without RandAugment, which has an error rate of 2.84.} 
        \label{tab:published_results}
        \vspace{-2.5em}
\end{table}

\vspace{-0.3em}
\subsection{Evaluation on Text Classification Datasets}
\label{sec:text_exp}
\vspace{-0.3em}

Next, we further evaluate \name in the language domain.
Moreover, in order to test whether \name can be combined with the success of unsupervised representation learning, such as BERT~\cite{devlin2018bert}, we further consider four initialization schemes: (a) random Transformer; (b) \bertbase; (c) \bertlarge; (d) \bertft: \bertlarge fine-tuned on in-domain unlabeled data\footnote{One exception is that we do not pursue \bertft on DBPedia as fine-tuning BERT on DBPedia does not yield further performance gain.  This is probably due to the fact that DBPedia is based on Wikipedia while BERT is already trained on the whole Wikipedia corpus.}.
Under each of these four initialization schemes, we compare the performances with and without \name.

\begin{table}[h!]
\vspace{-1em}
	\footnotesize
	\centering
	\renewcommand{\arraystretch}{0.7} 
	\begin{tabular}{lc|cccccc}
		\toprule
		\multicolumn{8}{c}{\bf Fully supervised baseline} \\
		\midrule
		\multicolumn{2}{c|}{\bf Datasets} & IMDb & Yelp-2 & Yelp-5 & Amazon-2 & Amazon-5 & DBpedia \\
		\multicolumn{2}{c|}{(\# Sup examples)} & (25k) & (560k) & (650k) & (3.6m) & (3m) & (560k) \\
		\midrule 
		\multicolumn{2}{l|}{Pre-BERT SOTA} & \emph{4.32} & 2.16 & 29.98 & 3.32 & 34.81 & 0.70 \\
		\multicolumn{2}{l|}{\bertlarge} & 4.51 & \emph{1.89} & \emph{29.32} & \emph{2.63} & \emph{34.17} & \emph{0.64} \\
		\midrule\midrule    
		\multicolumn{8}{c}{\bf Semi-supervised setting} \\
		\midrule
		\multirow{2}{*}{\bf Initialization} & \multirow{2}{*}{\bf \name}& IMDb & Yelp-2 & Yelp-5 & Amazon-2 & Amazon-5 & DBpedia  \\
		&& (20) & (20) & (2.5k) & (20) & (2.5k) & (140) \\
		\midrule
		\multirow{2}{*}{Random}    & \xmark & 43.27 & 40.25 & 50.80 & 45.39 & 55.70 & 41.14  \\
		& \cmark & 25.23 &  8.33 & 41.35 & 16.16 &  44.19 &  7.24\\
		\midrule
		\multirow{2}{*}{\bertbase} & \xmark & 18.40 & 13.60 & 41.00 & 26.75 & 44.09 &  2.58 \\
		& \cmark &  5.45 &  2.61 & 33.80 &  3.96 & 38.40 &  1.33 \\
		\midrule
		\multirow{2}{*}{\bertlarge}& \xmark & 11.72 & 10.55 & 38.90 & 15.54 & 42.30 & 1.68 \\
		& \cmark &  4.78 &  2.50 & 33.54 &  3.93 & 37.80 & 1.09 \\
		\midrule
		\multirow{2}{*}{\bertft}  & \xmark &  6.50 &  2.94 & 32.39 & 12.17  & 37.32  & - \\
		& \cmark & \textbf{4.20} & \textbf{2.05} & \textbf{32.08} & \textbf{3.50} & \textbf{37.12}  & -  \\
		\bottomrule
	\end{tabular}
	\caption{Error rates on text classification datasets. In the fully supervised settings, the pre-BERT SOTAs include ULMFiT~\cite{howard2018universal} for Yelp-2 and Yelp-5, DPCNN~\cite{johnson2017deep} for Amazon-2 and Amazon-5, Mixed VAT~\cite{sachan2018revisiting} for IMDb and DBPedia. All of our experiments use a sequence length of 512.}
	\label{tab:text_results}
\vspace{-1.5em}
\end{table}

The results are presented in Table \ref{tab:text_results} where we would like to emphasize three observations:
\begin{itemize}[leftmargin=*,itemsep=0em,topsep=0em]
\item First, even with very few labeled examples, \name can offer decent or even competitive performances compared to the SOTA model trained with full supervised data. 
Particularly, on binary sentiment analysis tasks, with only 20 supervised examples, \name outperforms the previous SOTA trained with full supervised data on IMDb and is competitive on Yelp-2 and Amazon-2.
\item Second, \name is complementary to transfer learning / representation learning. As we can see, when initialized with BERT and further finetuned on in-domain data, \name can still significantly reduce the error rate from $6.50$ to $4.20$ on IMDb.
\item Finally, we also note that for five-category sentiment classification tasks, there still exists a clear gap between \name with 500 labeled examples per class and BERT trained on the entire supervised set. Intuitively, five-category sentiment classifications are much more difficult than their binary counterparts.
This suggests a room for further improvement in the future.
\end{itemize}

\vspace{-0.3em}
\subsection{Scalability Test on the ImageNet Dataset}
\label{sec:imagenet_exp}
\vspace{-0.3em}

Then, to evaluate whether \name can scale to problems with a large scale and a higher difficulty, we now turn to the ImageNet dataset with ResNet-50 being the underlying architecture.
Specifically, we consider two experiment settings with different natures:
\begin{itemize}[leftmargin=*,itemsep=0em,topsep=0em,parsep=0em]
\item We use 10\% of the supervised data of ImageNet while using all other data as unlabeled data. 
As a result, the unlabeled exmaples are entirely in-domain.
\item In the second setting, we keep all images in ImageNet as supervised data. Then, we use the domain-relevance data filtering method to filter out 1.3M images from \jft . Hence, the unlabeled set is not necessarily in-domain.
\end{itemize}
The results are summarized in Table \ref{tab:imagenet}.
In both 10\% and the full data settings, \name consistently brings significant gains compared to the supervised baseline.
This shows \name is not only able to scale but also able to utilize out-of-domain unlabeled examples to improve model performance. 
In parallel to our work, S4L~\cite{zhai2019s} and CPC~\cite{henaff2019data} also show significant improvements on ImageNet.

\begin{table}[h!]
\vspace{-0.5em}
\footnotesize
\centering
\begin{tabular}{l|c|cc}
		\toprule
		{\bf Methods} & SSL & 10\% & 100\%  \\
		\midrule 
		ResNet-50 & \multirow{2}{*}{\xmark} &  55.09 / 77.26 & 77.28 / 93.73  \\
	    \ w. RandAugment &  & 58.84 / 80.56 & 78.43 / 94.37 \\
	    \midrule
		\name (RandAugment) & \cmark & \textbf{68.78 / 88.80} & \textbf{79.05 / 94.49} \\
		\bottomrule
	\end{tabular}
	\caption{Top-1 / top-5 accuracy on ImageNet with 10\% and 100\% of the labeled set. We use image size 224 and 331 for the 10\% and 100\% experiments respectively. } 
	\label{tab:imagenet}
\vspace{-1em}
\end{table}

\vspace{-0.3em}
\section{Related Work}
\vspace{-0.3em}
\label{sec:related}
Existing works in consistency training does make use of data augmentation~\cite{laine2016temporal, sajjadi2016regularization}; however, they only apply weak augmentation methods such as random translations and cropping.
In parallel to our work, ICT~\cite{verma2019interpolation} and MixMatch~\cite{berthelot2019mixmatch} also show improvements for semi-supervised learning. These methods employ mixup~\cite{zhang2017mixup} on top of simple augmentations such as flipping and cropping; instead, UDA emphasizes on the use of state-of-the-art data augmentations, leading to significantly better results on CIFAR-10 and SVHN.
In addition, UDA is also applicable to language domain and can also scale well to more challenging vision datasets, such as ImageNet.

Other works in the consistency training family mostly differ in how the noise is defined: Pseudo-ensemble~\cite{bachman2014learning} directly applies Gaussian noise and Dropout noise; 
VAT~\cite{miyato2018virtual,miyato2016adversarial} defines the noise by approximating the direction of change in the input space that the model is most sensitive to; Cross-view training~\cite{clark2018semi} masks out part of the input data. Apart from enforcing consistency on the input examples and the hidden representations, another line of research enforces consistency on the model parameter space. Works in this category include Mean Teacher~\cite{tarvainen2017mean}, fast-Stochastic Weight Averaging~\cite{athiwaratkun2018there} and Smooth Neighbors on Teacher Graphs~\cite{luo2018smooth}. For a complete version of related work, please refer to Appendix \ref{sec:apdx_related_work}.

\vspace{-0.3em}
\section{Conclusion}
\vspace{-0.3em}
In this paper, we show that data augmentation and semi-supervised learning are well connected: better data augmentation can lead to significantly better semi-supervised learning. Our method, \name, employs state-of-the-art data augmentation found in supervised learning to generate diverse and realistic noise and enforces the model to be consistent with respect to these noise. For text, UDA combines well with representation learning, e.g., BERT.
For vision, \name outperforms prior works by a clear margin and nearly matches the performance of the fully supervised models trained on the full labeled sets which are one order of magnitude larger.
We hope that UDA will encourage future research to transfer advanced supervised augmentation to semi-supervised setting for different tasks.

\section*{Acknowledgements}
We want to thank Hieu Pham, Adams Wei Yu, Zhilin Yang and Ekin Dogus Cubuk for their tireless help to the authors on different stages of this project and thank Colin Raffel for pointing out the connections between our work and previous works. We also would like to thank Olga Wichrowska, Barret Zoph, Jiateng Xie, Guokun Lai, Yulun Du, Chen Dan, David Berthelot, Avital Oliver, Trieu Trinh, Ran Zhao, Ola Spyra, Brandon Yang, Daiyi Peng, Andrew Dai, Samy Bengio, Jeff Dean and the Google Brain team for insightful discussions and support to the work. Lastly, we thank anonymous reviewers for their valueable feedbacks.

\section*{Broader Impact}
This work show that it is possible to achieve great performance with limited labeled data. Hence groups/institutes with limited budgets for annotating data may benefit from this research. To the best of our knowledge, nobody will be put at disadvantage from this research. 
Our method does not leverage biases in the data. Our tasks include standard benchmarks such as IMDb, CIFAR-10, SVHN and ImageNet.

\bibliographystyle{unsrt}
\bibliography{uda}

\begin{thebibliography}{10}

\bibitem{athiwaratkun2018there}
Ben Athiwaratkun, Marc Finzi, Pavel Izmailov, and Andrew~Gordon Wilson.
\newblock There are many consistent explanations of unlabeled data: Why you
  should average.
\newblock {\em ICLR}, 2019.

\bibitem{bachman2014learning}
Philip Bachman, Ouais Alsharif, and Doina Precup.
\newblock Learning with pseudo-ensembles.
\newblock In {\em Advances in Neural Information Processing Systems}, pages
  3365--3373, 2014.

\bibitem{berthelot2019mixmatch}
David Berthelot, Nicholas Carlini, Ian Goodfellow, Nicolas Papernot, Avital
  Oliver, and Colin Raffel.
\newblock Mixmatch: A holistic approach to semi-supervised learning.
\newblock {\em arXiv preprint arXiv:1905.02249}, 2019.

\bibitem{carmon2019unlabeled}
Yair Carmon, Aditi Raghunathan, Ludwig Schmidt, Percy Liang, and John~C Duchi.
\newblock Unlabeled data improves adversarial robustness.
\newblock {\em arXiv preprint arXiv:1905.13736}, 2019.

\bibitem{chapelle2009semi}
Olivier Chapelle, Bernhard Scholkopf, and Alexander Zien.
\newblock Semi-supervised learning (chapelle, o. et al., eds.; 2006)[book
  reviews].
\newblock {\em IEEE Transactions on Neural Networks}, 20(3):542--542, 2009.

\bibitem{chollet2017xception}
Fran{\c{c}}ois Chollet.
\newblock Xception: Deep learning with depthwise separable convolutions.
\newblock In {\em Proceedings of the IEEE conference on computer vision and
  pattern recognition}, pages 1251--1258, 2017.

\bibitem{clark2018semi}
Kevin Clark, Minh-Thang Luong, Christopher~D Manning, and Quoc~V Le.
\newblock Semi-supervised sequence modeling with cross-view training.
\newblock {\em arXiv preprint arXiv:1809.08370}, 2018.

\bibitem{collobert2008unified}
Ronan Collobert and Jason Weston.
\newblock A unified architecture for natural language processing: Deep neural
  networks with multitask learning.
\newblock In {\em Proceedings of the 25th international conference on Machine
  learning}, pages 160--167. ACM, 2008.

\bibitem{cubuk2018autoaugment}
Ekin~D Cubuk, Barret Zoph, Dandelion Mane, Vijay Vasudevan, and Quoc~V Le.
\newblock Autoaugment: Learning augmentation policies from data.
\newblock {\em arXiv preprint arXiv:1805.09501}, 2018.

\bibitem{cubuk2019RandAugment}
Ekin~D Cubuk, Barret Zoph, Jonathon Shlens, and Quoc~V Le.
\newblock Randaugment: Practical data augmentation with no separate search.
\newblock {\em arXiv preprint arXiv:1909.13719}, 2019.

\bibitem{dai2015semi}
Andrew~M Dai and Quoc~V Le.
\newblock Semi-supervised sequence learning.
\newblock In {\em Advances in neural information processing systems}, pages
  3079--3087, 2015.

\bibitem{dai2017good}
Zihang Dai, Zhilin Yang, Fan Yang, William~W Cohen, and Ruslan~R Salakhutdinov.
\newblock Good semi-supervised learning that requires a bad gan.
\newblock In {\em Advances in Neural Information Processing Systems}, pages
  6510--6520, 2017.

\bibitem{deng2009imagenet}
Jia Deng, Wei Dong, Richard Socher, Li-Jia Li, Kai Li, and Li~Fei-Fei.
\newblock Imagenet: A large-scale hierarchical image database.
\newblock In {\em 2009 IEEE conference on computer vision and pattern
  recognition}, pages 248--255. Ieee, 2009.

\bibitem{devlin2018bert}
Jacob Devlin, Ming-Wei Chang, Kenton Lee, and Kristina Toutanova.
\newblock Bert: Pre-training of deep bidirectional transformers for language
  understanding.
\newblock {\em arXiv preprint arXiv:1810.04805}, 2018.

\bibitem{edunov2018understanding}
Sergey Edunov, Myle Ott, Michael Auli, and David Grangier.
\newblock Understanding back-translation at scale.
\newblock {\em arXiv preprint arXiv:1808.09381}, 2018.

\bibitem{grandvalet2005semi}
Yves Grandvalet and Yoshua Bengio.
\newblock Semi-supervised learning by entropy minimization.
\newblock In {\em Advances in neural information processing systems}, pages
  529--536, 2005.

\bibitem{hannun2014deep}
Awni Hannun, Carl Case, Jared Casper, Bryan Catanzaro, Greg Diamos, Erich
  Elsen, Ryan Prenger, Sanjeev Satheesh, Shubho Sengupta, Adam Coates, et~al.
\newblock Deep speech: Scaling up end-to-end speech recognition.
\newblock {\em arXiv preprint arXiv:1412.5567}, 2014.

\bibitem{he2016deep}
Kaiming He, Xiangyu Zhang, Shaoqing Ren, and Jian Sun.
\newblock Deep residual learning for image recognition.
\newblock In {\em Proceedings of the IEEE conference on computer vision and
  pattern recognition}, pages 770--778, 2016.

\bibitem{he2018sequence}
Xuanli He, Gholamreza Haffari, and Mohammad Norouzi.
\newblock Sequence to sequence mixture model for diverse machine translation.
\newblock {\em arXiv preprint arXiv:1810.07391}, 2018.

\bibitem{henaff2019data}
Olivier~J H{\'e}naff, Ali Razavi, Carl Doersch, SM~Eslami, and Aaron van~den
  Oord.
\newblock Data-efficient image recognition with contrastive predictive coding.
\newblock {\em arXiv preprint arXiv:1905.09272}, 2019.

\bibitem{hernandez2018data}
Alex Hern{\'a}ndez-Garc{\'\i}a and Peter K{\"o}nig.
\newblock Data augmentation instead of explicit regularization.
\newblock {\em arXiv preprint arXiv:1806.03852}, 2018.

\bibitem{hinton2015distilling}
Geoffrey Hinton, Oriol Vinyals, and Jeff Dean.
\newblock Distilling the knowledge in a neural network.
\newblock {\em arXiv preprint arXiv:1503.02531}, 2015.

\bibitem{howard2018universal}
Jeremy Howard and Sebastian Ruder.
\newblock Universal language model fine-tuning for text classification.
\newblock In {\em Proceedings of the 56th Annual Meeting of the Association for
  Computational Linguistics (Volume 1: Long Papers)}, volume~1, pages 328--339,
  2018.

\bibitem{hu2017learning}
Weihua Hu, Takeru Miyato, Seiya Tokui, Eiichi Matsumoto, and Masashi Sugiyama.
\newblock Learning discrete representations via information maximizing
  self-augmented training.
\newblock In {\em Proceedings of the 34th International Conference on Machine
  Learning-Volume 70}, pages 1558--1567. JMLR. org, 2017.

\bibitem{jackson2019semi}
Jacob Jackson and John Schulman.
\newblock Semi-supervised learning by label gradient alignment.
\newblock {\em arXiv preprint arXiv:1902.02336}, 2019.

\bibitem{johnson2017deep}
Rie Johnson and Tong Zhang.
\newblock Deep pyramid convolutional neural networks for text categorization.
\newblock In {\em Proceedings of the 55th Annual Meeting of the Association for
  Computational Linguistics (Volume 1: Long Papers)}, volume~1, pages 562--570,
  2017.

\bibitem{kingma2014semi}
Durk~P Kingma, Shakir Mohamed, Danilo~Jimenez Rezende, and Max Welling.
\newblock Semi-supervised learning with deep generative models.
\newblock In {\em Advances in neural information processing systems}, pages
  3581--3589, 2014.

\bibitem{kipf2016semi}
Thomas~N Kipf and Max Welling.
\newblock Semi-supervised classification with graph convolutional networks.
\newblock {\em arXiv preprint arXiv:1609.02907}, 2016.

\bibitem{kool2019stochastic}
Wouter Kool, Herke van Hoof, and Max Welling.
\newblock Stochastic beams and where to find them: The gumbel-top-k trick for
  sampling sequences without replacement.
\newblock {\em arXiv preprint arXiv:1903.06059}, 2019.

\bibitem{krizhevsky2009learning}
Alex Krizhevsky and Geoffrey Hinton.
\newblock Learning multiple layers of features from tiny images.
\newblock Technical report, Citeseer, 2009.

\bibitem{krizhevsky2012imagenet}
Alex Krizhevsky, Ilya Sutskever, and Geoffrey~E Hinton.
\newblock Imagenet classification with deep convolutional neural networks.
\newblock In {\em Advances in neural information processing systems}, pages
  1097--1105, 2012.

\bibitem{laine2016temporal}
Samuli Laine and Timo Aila.
\newblock Temporal ensembling for semi-supervised learning.
\newblock {\em arXiv preprint arXiv:1610.02242}, 2016.

\bibitem{lee2013pseudo}
Dong-Hyun Lee.
\newblock Pseudo-label: The simple and efficient semi-supervised learning
  method for deep neural networks.
\newblock In {\em Workshop on Challenges in Representation Learning, ICML},
  volume~3, page~2, 2013.

\bibitem{liang2018learning}
Davis Liang, Zhiheng Huang, and Zachary~C Lipton.
\newblock Learning noise-invariant representations for robust speech
  recognition.
\newblock In {\em 2018 IEEE Spoken Language Technology Workshop (SLT)}, pages
  56--63. IEEE, 2018.

\bibitem{luo2018smooth}
Yucen Luo, Jun Zhu, Mengxi Li, Yong Ren, and Bo~Zhang.
\newblock Smooth neighbors on teacher graphs for semi-supervised learning.
\newblock In {\em Proceedings of the IEEE Conference on Computer Vision and
  Pattern Recognition}, pages 8896--8905, 2018.

\bibitem{maaloe2016auxiliary}
Lars Maal{\o}e, Casper~Kaae S{\o}nderby, S{\o}ren~Kaae S{\o}nderby, and Ole
  Winther.
\newblock Auxiliary deep generative models.
\newblock {\em arXiv preprint arXiv:1602.05473}, 2016.

\bibitem{maas2011learning}
Andrew~L Maas, Raymond~E Daly, Peter~T Pham, Dan Huang, Andrew~Y Ng, and
  Christopher Potts.
\newblock Learning word vectors for sentiment analysis.
\newblock In {\em Proceedings of the 49th annual meeting of the association for
  computational linguistics: Human language technologies-volume 1}, pages
  142--150. Association for Computational Linguistics, 2011.

\bibitem{mcauley2015image}
Julian McAuley, Christopher Targett, Qinfeng Shi, and Anton Van Den~Hengel.
\newblock Image-based recommendations on styles and substitutes.
\newblock In {\em Proceedings of the 38th International ACM SIGIR Conference on
  Research and Development in Information Retrieval}, pages 43--52. ACM, 2015.

\bibitem{mikolov2013distributed}
Tomas Mikolov, Ilya Sutskever, Kai Chen, Greg~S Corrado, and Jeff Dean.
\newblock Distributed representations of words and phrases and their
  compositionality.
\newblock In {\em Advances in neural information processing systems}, pages
  3111--3119, 2013.

\bibitem{miyato2016adversarial}
Takeru Miyato, Andrew~M Dai, and Ian Goodfellow.
\newblock Adversarial training methods for semi-supervised text classification.
\newblock {\em arXiv preprint arXiv:1605.07725}, 2016.

\bibitem{miyato2018virtual}
Takeru Miyato, Shin-ichi Maeda, Shin Ishii, and Masanori Koyama.
\newblock Virtual adversarial training: a regularization method for supervised
  and semi-supervised learning.
\newblock {\em IEEE transactions on pattern analysis and machine intelligence},
  2018.

\bibitem{najafi2019robustness}
Amir Najafi, Shin-ichi Maeda, Masanori Koyama, and Takeru Miyato.
\newblock Robustness to adversarial perturbations in learning from incomplete
  data.
\newblock {\em arXiv preprint arXiv:1905.13021}, 2019.

\bibitem{netzer2011reading}
Yuval Netzer, Tao Wang, Adam Coates, Alessandro Bissacco, Bo~Wu, and Andrew~Y
  Ng.
\newblock Reading digits in natural images with unsupervised feature learning.
\newblock 2011.

\bibitem{oliver2018realistic}
Avital Oliver, Augustus Odena, Colin~A Raffel, Ekin~Dogus Cubuk, and Ian
  Goodfellow.
\newblock Realistic evaluation of deep semi-supervised learning algorithms.
\newblock In {\em Advances in Neural Information Processing Systems}, pages
  3235--3246, 2018.

\bibitem{park2019specaugment}
Daniel~S Park, William Chan, Yu~Zhang, Chung-Cheng Chiu, Barret Zoph, Ekin~D
  Cubuk, and Quoc~V Le.
\newblock Specaugment: A simple data augmentation method for automatic speech
  recognition.
\newblock {\em arXiv preprint arXiv:1904.08779}, 2019.

\bibitem{pennington2014glove}
Jeffrey Pennington, Richard Socher, and Christopher Manning.
\newblock Glove: Global vectors for word representation.
\newblock In {\em Proceedings of the 2014 conference on empirical methods in
  natural language processing (EMNLP)}, pages 1532--1543, 2014.

\bibitem{peters2018deep}
Matthew~E Peters, Mark Neumann, Mohit Iyyer, Matt Gardner, Christopher Clark,
  Kenton Lee, and Luke Zettlemoyer.
\newblock Deep contextualized word representations.
\newblock {\em arXiv preprint arXiv:1802.05365}, 2018.

\bibitem{radford2018improving}
Alec Radford, Karthik Narasimhan, Tim Salimans, and Ilya Sutskever.
\newblock Improving language understanding by generative pre-training.
\newblock {\em URL https://s3-us-west-2. amazonaws.
  com/openai-assets/research-covers/languageunsupervised/language understanding
  paper. pdf}, 2018.

\bibitem{rasmus2015semi}
Antti Rasmus, Mathias Berglund, Mikko Honkala, Harri Valpola, and Tapani Raiko.
\newblock Semi-supervised learning with ladder networks.
\newblock In {\em Advances in neural information processing systems}, pages
  3546--3554, 2015.

\bibitem{sachan2018revisiting}
Devendra~Singh Sachan, Manzil Zaheer, and Ruslan Salakhutdinov.
\newblock Revisiting lstm networks for semi-supervised text classification via
  mixed objective function.
\newblock 2018.

\bibitem{sajjadi2016regularization}
Mehdi Sajjadi, Mehran Javanmardi, and Tolga Tasdizen.
\newblock Regularization with stochastic transformations and perturbations for
  deep semi-supervised learning.
\newblock In {\em Advances in Neural Information Processing Systems}, pages
  1163--1171, 2016.

\bibitem{Salazar2018Invariant}
Julian Salazar, Davis Liang, Zhiheng Huang, and Zachary~C Lipton.
\newblock Invariant representation learning for robust deep networks.
\newblock In {\em Workshop on Integration of Deep Learning Theories, NeurIPS},
  2018.

\bibitem{salimans2016improved}
Tim Salimans, Ian Goodfellow, Wojciech Zaremba, Vicki Cheung, Alec Radford, and
  Xi~Chen.
\newblock Improved techniques for training gans.
\newblock In {\em Advances in neural information processing systems}, pages
  2234--2242, 2016.

\bibitem{sennrich2015improving}
Rico Sennrich, Barry Haddow, and Alexandra Birch.
\newblock Improving neural machine translation models with monolingual data.
\newblock {\em arXiv preprint arXiv:1511.06709}, 2015.

\bibitem{shen2019mixture}
Tianxiao Shen, Myle Ott, Michael Auli, and Marc'Aurelio Ranzato.
\newblock Mixture models for diverse machine translation: Tricks of the trade.
\newblock {\em arXiv preprint arXiv:1902.07816}, 2019.

\bibitem{simard1998transformation}
Patrice~Y Simard, Yann~A LeCun, John~S Denker, and Bernard Victorri.
\newblock Transformation invariance in pattern recognition—tangent distance
  and tangent propagation.
\newblock In {\em Neural networks: tricks of the trade}, pages 239--274.
  Springer, 1998.

\bibitem{stanforth2019labels}
Robert Stanforth, Alhussein Fawzi, Pushmeet Kohli, et~al.
\newblock Are labels required for improving adversarial robustness?
\newblock {\em arXiv preprint arXiv:1905.13725}, 2019.

\bibitem{tarvainen2017mean}
Antti Tarvainen and Harri Valpola.
\newblock Mean teachers are better role models: Weight-averaged consistency
  targets improve semi-supervised deep learning results.
\newblock In {\em Advances in neural information processing systems}, pages
  1195--1204, 2017.

\bibitem{trinh2019selfie}
Trieu~H Trinh, Minh-Thang Luong, and Quoc~V Le.
\newblock Selfie: Self-supervised pretraining for image embedding.
\newblock {\em arXiv preprint arXiv:1906.02940}, 2019.

\bibitem{verma2019interpolation}
Vikas Verma, Alex Lamb, Juho Kannala, Yoshua Bengio, and David Lopez-Paz.
\newblock Interpolation consistency training for semi-supervised learning.
\newblock {\em arXiv preprint arXiv:1903.03825}, 2019.

\bibitem{wang2018switchout}
Xinyi Wang, Hieu Pham, Zihang Dai, and Graham Neubig.
\newblock Switchout: an efficient data augmentation algorithm for neural
  machine translation.
\newblock {\em arXiv preprint arXiv:1808.07512}, 2018.

\bibitem{weston2012deep}
Jason Weston, Fr{\'e}d{\'e}ric Ratle, Hossein Mobahi, and Ronan Collobert.
\newblock Deep learning via semi-supervised embedding.
\newblock In {\em Neural Networks: Tricks of the Trade}, pages 639--655.
  Springer, 2012.

\bibitem{yang2016revisiting}
Zhilin Yang, William~W Cohen, and Ruslan Salakhutdinov.
\newblock Revisiting semi-supervised learning with graph embeddings.
\newblock {\em arXiv preprint arXiv:1603.08861}, 2016.

\bibitem{yang2017semi}
Zhilin Yang, Junjie Hu, Ruslan Salakhutdinov, and William~W Cohen.
\newblock Semi-supervised qa with generative domain-adaptive nets.
\newblock {\em arXiv preprint arXiv:1702.02206}, 2017.

\bibitem{ye2019unsupervised}
Mang Ye, Xu~Zhang, Pong~C Yuen, and Shih-Fu Chang.
\newblock Unsupervised embedding learning via invariant and spreading instance
  feature.
\newblock In {\em Proceedings of the IEEE Conference on Computer Vision and
  Pattern Recognition}, pages 6210--6219, 2019.

\bibitem{yu2018qanet}
Adams~Wei Yu, David Dohan, Minh-Thang Luong, Rui Zhao, Kai Chen, Mohammad
  Norouzi, and Quoc~V Le.
\newblock Qanet: Combining local convolution with global self-attention for
  reading comprehension.
\newblock {\em arXiv preprint arXiv:1804.09541}, 2018.

\bibitem{zagoruyko2016wide}
Sergey Zagoruyko and Nikos Komodakis.
\newblock Wide residual networks.
\newblock {\em BMVC}, 2016.

\bibitem{zhai2019adversarially}
Runtian Zhai, Tianle Cai, Di~He, Chen Dan, Kun He, John Hopcroft, and Liwei
  Wang.
\newblock Adversarially robust generalization just requires more unlabeled
  data.
\newblock {\em arXiv preprint arXiv:1906.00555}, 2019.

\bibitem{zhai2019s}
Xiaohua Zhai, Avital Oliver, Alexander Kolesnikov, and Lucas Beyer.
\newblock S$^\mathbf{4}$l: Self-supervised semi-supervised learning.
\newblock In {\em Proceedings of the IEEE international conference on computer
  vision}, 2019.

\bibitem{zhang2017mixup}
Hongyi Zhang, Moustapha Cisse, Yann~N Dauphin, and David Lopez-Paz.
\newblock mixup: Beyond empirical risk minimization.
\newblock {\em arXiv preprint arXiv:1710.09412}, 2017.

\bibitem{zhang2015character}
Xiang Zhang, Junbo Zhao, and Yann LeCun.
\newblock Character-level convolutional networks for text classification.
\newblock In {\em Advances in neural information processing systems}, pages
  649--657, 2015.

\bibitem{zhu2003semi}
Xiaojin Zhu, Zoubin Ghahramani, and John~D Lafferty.
\newblock Semi-supervised learning using gaussian fields and harmonic
  functions.
\newblock In {\em Proceedings of the 20th International conference on Machine
  learning (ICML-03)}, pages 912--919, 2003.

\end{thebibliography}

\clearpage

\appendix
\section{Extended Method Details}
In this section, we present some additional details used in our method. We introduce Training Signal Annealing in Appendix \ref{sec:tsa} and details for augmentation strategies in Appendix \ref{sec:apdx_augmentations}. 
\subsection{Training Signal Annealing for Low-data Regime}
\label{sec:tsa}
In semi-supervised learning, we often encounter a situation where there is a huge gap between the amount of unlabeled data and that of labeled data.
Hence, the model often quickly overfits the limited amount of labeled data while still underfitting the unlabeled data.
To tackle this difficulty, we introduce a new training technique, called Training Signal Annealing (TSA), which
gradually releases the ``training signals'' of the labeled examples as training progresses.
Intuitively, we only utilize a labeled example if the model's confidence on that example is lower than a predefined threshold which increases according to a schedule.
Specifically, at training step $t$, if the model's predicted probability for the correct category $p_\theta(y^* \mid x)$
is higher than a threshold $\eta_t$, we remove that example from the loss function.
Suppose $K$ is the number of categories, by gradually increasing $\eta_t$ from $\frac{1}{K}$ to $1$, the threshold $\eta_t$ serves as a ceiling to prevent over-training on easy labeled examples. 

We consider three increasing schedules of $\eta_t$ with different application scenarios.
Let $T$ be the total number of training steps, the three schedules are shown in Figure \ref{fig:tsa}.
Intuitively, when the model is prone to overfit, e.g., when the problem is relatively easy or the number of labeled examples is very limited, the exp-schedule is most suitable as the supervised signal is mostly released at the end of training. In contrast, when the model is less likely to overfit (e.g., when we have abundant labeled examples or when the model employs effective regularization), the log-schedule can serve well. 

\begin{figure}[h!]
    \centering
    \begin{subfigure}[t]{0.32\textwidth}
        \centering
    \includegraphics[width=\textwidth]{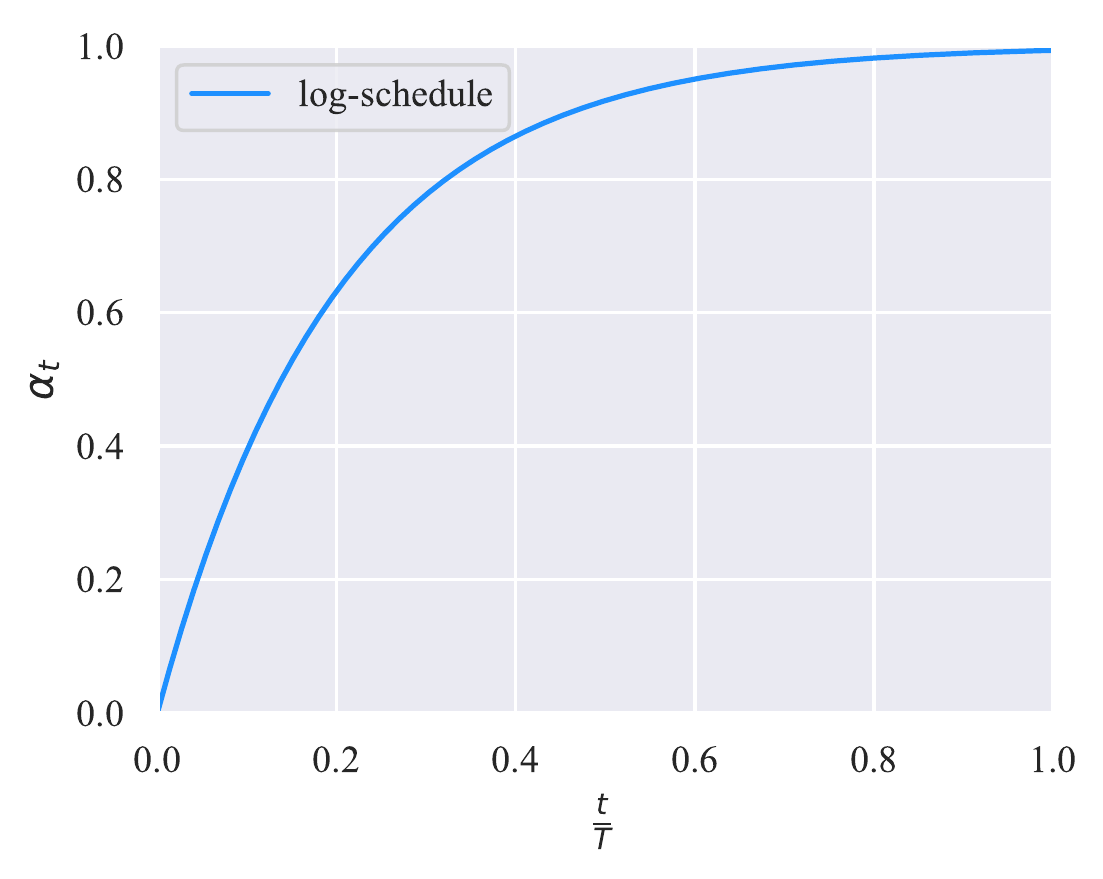}
    \end{subfigure}
    ~ 
    \begin{subfigure}[t]{0.32\textwidth}
        \centering
    \includegraphics[width=\textwidth]{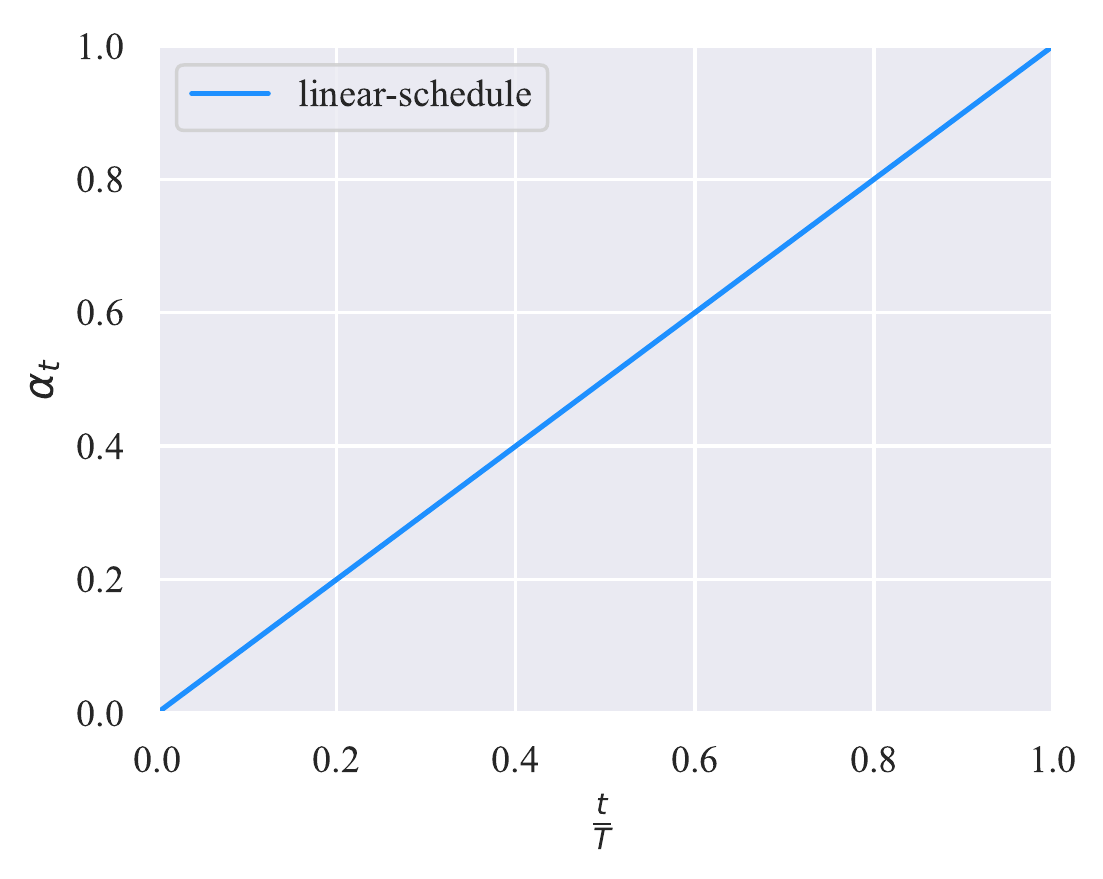}
    \end{subfigure}
    ~ 
    \begin{subfigure}[t]{0.32\textwidth}
        \centering
    \includegraphics[width=\textwidth]{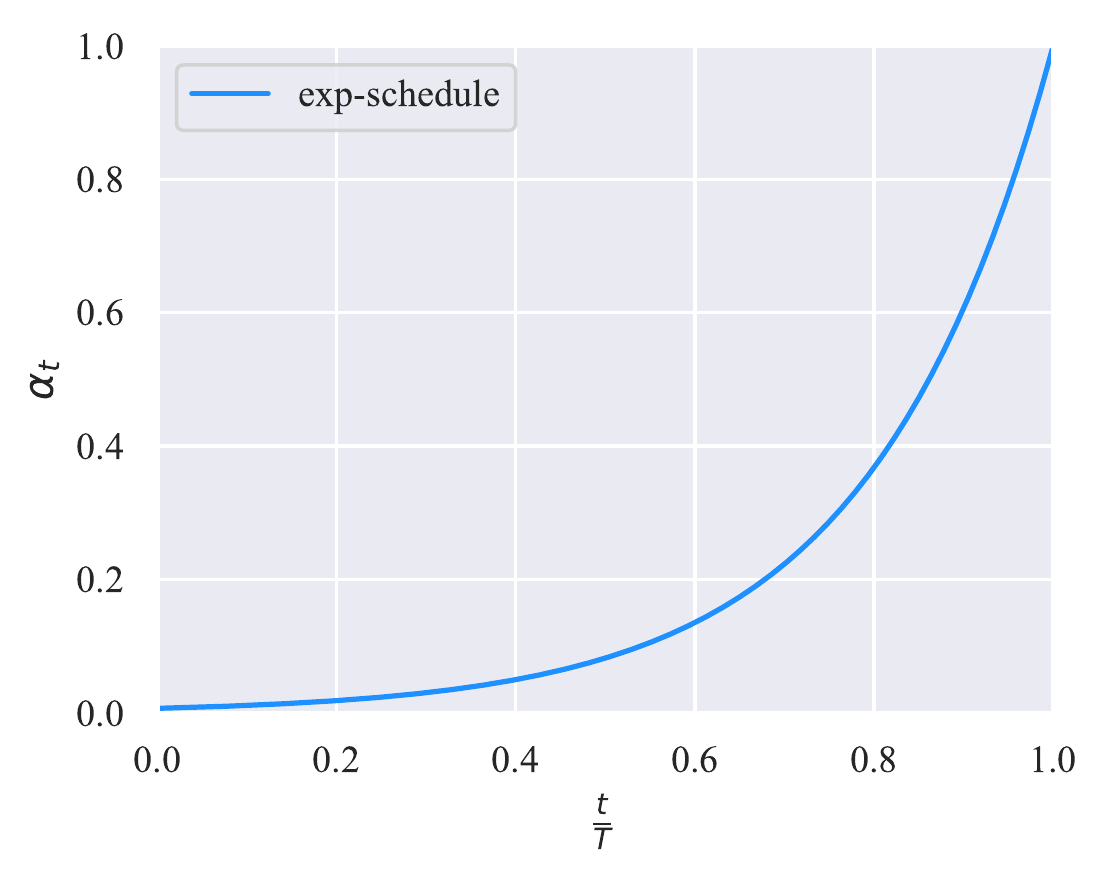}
    \end{subfigure}
\caption{Three schedules of TSA. We set $\eta_t= \alpha_t * (1-\frac{1}{K}) + \frac{1}{K}$. $\alpha_t$ is set to $1-\exp(-\frac{t}{T} * 5)$, $\frac{t}{T}$ and $\exp((\frac{t}{T} - 1) * 5)$ for the log, linear and exp schedules.}
\label{fig:tsa}
\end{figure}

\subsection{Extended Augmentation Strategies for Different Tasks}
\label{sec:apdx_augmentations}

\textbf{Discussion on Trade-off Between Diversity and Validity for Data Augmentation. }
Despite that state-of-the-art data augmentation methods can generate diverse and valid augmented examples as discussed in section \ref{sec:uda}, there is a trade-off between diversity and validity since diversity is achieved by changing a part of the original example, naturally leading to the risk of altering the ground-truth label. 
We find it beneficial to tune the trade-off between diversity and validity for data augmentation methods.
For text classification, we tune the temperature of random sampling. On the one hand, when we use a temperature of $0$, decoding by random sampling degenerates into greedy decoding and generates perfectly valid but identical paraphrases. 
On the other hand, when we use a temperature of $1$, random sampling generates very diverse but barely readable paraphrases. We find that setting the Softmax temperature to $0.7, 0.8$ or $0.9$ leads to the best performances. 

\textbf{RandAugment Details.}
In our implementation of RandAugment, each sub-policy is composed of two operations, where each operation is represented by the transformation name, probability, and magnitude that is specific to that operation. For example, a sub-policy can be [(Sharpness, 0.6, 2), (Posterize, 0.3, 9)]. 

For each operation, we randomly sample a transformation from $15$ possible transformations, a magnitude in $[1, 10)$ and fix the probability to $0.5$.  Specifically, we sample from the following $15$ transformations: Invert, Cutout, Sharpness, AutoContrast, Posterize, ShearX, TranslateX, TranslateY, ShearY, Rotate, Equalize, Contrast, Color, Solarize, Brightness.
We find this setting to work well in our first try and did not tune the magnitude  range and the probability. Tuning these hyperparameters might result in further gains in accuracy.

\textbf{TF-IDF based word replacing Details.}
Ideally, we would like the augmentation method to generate both diverse and valid examples. Hence, the augmentation is designed to retain keywords and replace uninformative words with other uninformative words. We use BERT's word tokenizer since BERT first tokenizes sentences into a sequence of words and then tokenize words into subwords although the model uses subwords as input. 

Specifically, Suppose $\mathrm{IDF}(w)$ is the IDF score for word $w$ computed on the whole corpus, and $\mathrm{TF}(w)$ is the TF score for word $w$ in a sentence. We compute the TF-IDF score as $\mathrm{TFIDF}(w)=\mathrm{TF}(w) \mathrm{IDF}(w)$. Suppose the maximum TF-IDF score in a sentence $x$ is $C=\max_{i} \mathrm{TFIDF}(x_i)$. To make the probability of having a word replaced to  negatively correlate with its TF-IDF score, we set the probability to  $\min(p (C - \mathrm{TFIDF}(x_i)) / Z, 1)$, where $p$ is a hyperparameter that controls the magnitude of the augmentation and $Z=\sum_i (C - \mathrm{TFIDF}(x_i)) / |x|$ is the average score. $p$ is set to 0.7 for experiments on DBPedia.

When a word is replaced, we sample another word from the whole vocabulary for the replacement. Intuitively, the sampled words should not be keywords to prevent changing the ground-truth labels of the sentence. To measure if a word is keyword, we compute a score of each word on the whole corpus. Specifically, we compute the score as  $S(w)=\mathrm{freq}(w)\mathrm{IDF}(w)$ where  $\mathrm{freq}(w)$ is the frequency of word $w$ on the whole corpus.  We set the probability of sampling word $w$ as $(\max_{w'}S(w') - S(w)) / Z'$ where $Z'=\sum_w \max_{w'}S(w') - S(w)$ is a normalization term.

\section{Extended Experiments}
\label{apdx:exp}

\subsection{Ablation Studies}

\paragraph{Ablation Studies for Unlabeled Data Size}
Here we present an ablation study for unlabeled data sizes. As shown in Table \ref{tab:cifar_10_unlabeled data} and Table \ref{tab:svhn_unlabeled data}, given the same number of labeled examples, reducing the number of unsupervised examples clearly leads to worse performance. 
In fact, having abundant unsupervised examples is more important than having more labeled examples  since reducing the unlabeled data amount leads to worse performance than reducing the labeled data by the same ratio. 
\begin{table}[H]
\centering
\small
\begin{tabular}{l|ccccc}
\toprule
\textbf{\# Unsup / \# Sup} & 250 & 500 & 1,000 & 2,000 & 4,000 \\
\midrule
50,000 & 5.43 $\pm$ 0.96 & 4.80 $\pm$ 0.09 & 4.75 $\pm$ 0.10 & 4.73 $\pm$ 0.14 & 4.32 $\pm$ 0.08 \\
20,000 & 11.01 $\pm$ 1.01 & 9.46 $\pm$ 0.14 & 8.57 $\pm$ 0.14 & 7.65 $\pm$ 0.17 & 7.31 $\pm$ 0.24 \\
10,000 & 23.17 $\pm$ 0.71 & 18.43 $\pm$ 0.43 & 15.46 $\pm$ 0.58 & 12.52 $\pm$ 0.13 & 10.32 $\pm$ 0.20 \\
5,000 & 35.41 $\pm$ 0.75 & 28.35 $\pm$ 0.60 & 22.06 $\pm$ 0.71 & 17.36 $\pm$ 0.15 & 13.19 $\pm$ 0.12 \\
\bottomrule
\end{tabular}
\caption{Error rate (\%) for CIFAR-10 with different amounts of labeled data and unlabeled data.} 
\label{tab:cifar_10_unlabeled data}
\end{table}

\begin{table}[H]
\centering
\small
\begin{tabular}{l|ccccc}
\toprule
\textbf{\# Unsup / \# Sup} & 250 & 500 & 1,000 & 2,000 & 4,000 \\
\midrule
73,257 & 2.72 $\pm$ 0.40 & 2.27 $\pm$ 0.09 & 2.23 $\pm$ 0.07 & 2.20 $\pm$ 0.06 & 2.28 $\pm$ 0.10  \\
20,000 & 5.59 $\pm$ 0.74 &  4.43 $\pm$ 0.15  & 3.81 $\pm$ 0.11 & 3.86 $\pm$ 0.14 & 3.64 $\pm$ 0.20 \\
10,000 & 17.13 $\pm$ 12.85 & 7.59 $\pm$ 1.01 & 5.76 $\pm$ 0.29 & 5.17 $\pm$ 0.12 & 5.40 $\pm$ 0.12  \\
5,000 & 31.58 $\pm$ 7.39 & 12.66 $\pm$ 0.81 &  6.28 $\pm$ 0.25 & 8.35 $\pm$ 0.36 & 7.76 $\pm$ 0.28 \\
\bottomrule
\end{tabular}
\caption{Error rate (\%) for SVHN with different amounts of labeled data and unlabeled data.} 
\label{tab:svhn_unlabeled data}
\end{table}

\paragraph{Ablations Studies on RandAugment} 

We hypothesize that the success of RandAugment should be credited to the diversity of the augmentation transformations, since RandAugment works very well for multiple different datasets while it does not require a search algorithm to find out the most effective policies. To verify this hypothesis, we test UDA's performance when we restrict the number of possible transformations used in RandAugment. As shown in Figure \ref{fig:randaugment_op}, the performance gradually improves as we use more augmentation transformations.
\begin{figure}[h!]
    \centering
    \footnotesize
    \includegraphics[width=0.4\textwidth]{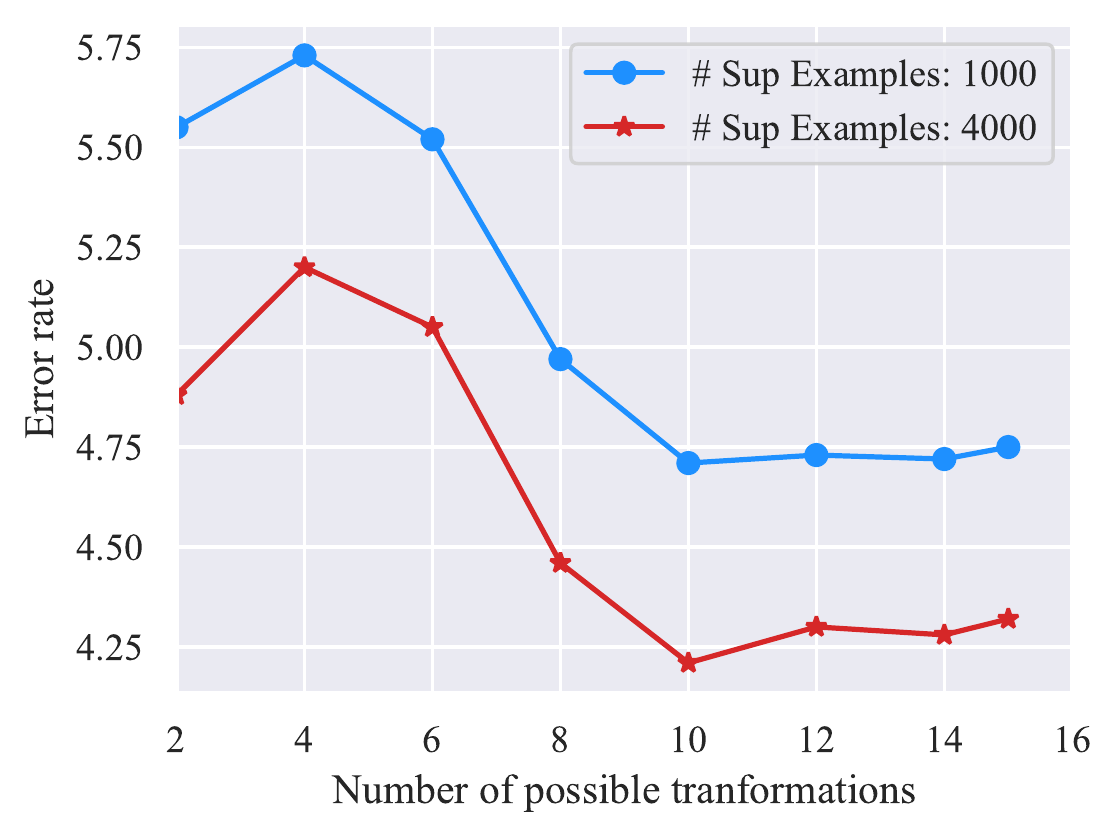}
    \caption{Error rate of UDA on CIFAR-10 with different numbers of possible transformations in RandAugment. UDA achieves lower error rate when we increase the number of possible transformations, which demonstrates the importance of a rich set of augmentation transformations.}
    \label{fig:randaugment_op}
\end{figure}

\paragraph{Ablation Studies for TSA}
We study the effect of TSA on Yelp-5 where we have $2.5$k labeled examples and $6$m unlabeled examples. We use a randomly initialized transformer in this study to rule out factors of having a pre-trained representation.

As shown in Table \ref{tab:ablation_TSA}, on Yelp-5, where there is a lot more unlabeled data than labeled data,
TSA reduces the error rate from $50.81$ to $41.35$ when compared to the baseline without TSA. More specifically, the best performance is achieved when we choose to postpone releasing the supervised training signal to the end of the training, i.e, exp-schedule leads to the best performance.

\begin{table}[h!]
\centering
\footnotesize
	\begin{tabular}{l|cc}
		\toprule
		 \textbf{TSA schedule} & Yelp-5    \\
		 \midrule 
		 \xmark      & 50.81 \\
		 log-schedule  & 49.06  \\
		 linear-schedule  	 & 45.41  \\
		 exp-schedule  	 & \textbf{41.35} \\
		\bottomrule
	\end{tabular}
	\caption{Ablation study for Training Signal Annealing (TSA) on Yelp-5 and CIFAR-10. The shown numbers are error rates.}
	\label{tab:ablation_TSA}
\end{table}

\subsection{More Results on CIFAR-10, SVHN and Text Classification Datasets}
\label{sec:appdx_more_results}
\paragraph{Results with varied label set sizes on CIFAR-10}  In Table \ref{tab:cifar10_vary_sup}, we show results for compared methods of Figure \ref{fig:cifar_vary_sup} and results of Pseudo-Label~\cite{lee2013pseudo}, $\Pi$-Model~\cite{laine2016temporal}, Mean Teacher~\cite{tarvainen2017mean}. Fully supervised learning using 50,000 examples achieves an error rate of 4.23 and 5.36  with or without RandAugment. The performance of the baseline models are reported by MixMatch~\cite{berthelot2019mixmatch}. 

To make sure that the performance reported by MixMatch and our results are comparable, we reimplement MixMatch in our codebase and find that the results in the original paper is comparable but slightly better than our reimplementation, which results in a more competitive comparison for UDA. For example, our reimplementation of MixMatch achieves an error rate of 7.00 $\pm$ 0.59 and 7.39 $\pm$ 0.11 with 4,000 and 2,000 examples. 
\begin{table}[H]
\centering
\small
\begin{tabular}{l|ccccc}
\toprule
\textbf{Methods / \# Sup} & 250 & 500 & 1,000 & 2,000 & 4,000 \\
\midrule
Pseudo-Label & 49.98 $\pm$ 1.17 & 40.55 $\pm$ 1.70 & 30.91 $\pm$ 1.73 & 21.96 $\pm$ 0.42 & 16.21 $\pm$ 0.11 \\
$\Pi$-Model & 53.02 $\pm$ 2.05 & 41.82 $\pm$ 1.52 & 31.53 $\pm$ 0.98 & 23.07 $\pm$ 0.66 & 17.41 $\pm$ 0.37 \\
Mean Teacher & 47.32 $\pm$ 4.71 & 42.01 $\pm$ 5.86 & 17.32 $\pm$ 4.00 & 12.17 $\pm$ 0.22 & 10.36 $\pm$ 0.25 \\
VAT & 36.03 $\pm$ 2.82 & 26.11 $\pm$ 1.52 & 18.68 $\pm$ 0.40 & 14.40 $\pm$ 0.15 & 11.05 $\pm$ 0.31 \\
MixMatch & 11.08 $\pm$ 0.87 & 9.65 $\pm$ 0.94 & 7.75 $\pm$ 0.32 & 7.03 $\pm$ 0.15 & 6.24 $\pm$ 0.06 \\
UDA (RandAugment) & \textbf{5.43 $\pm$ 0.96} & \textbf{4.80 $\pm$ 0.09} &  \textbf{4.75 $\pm$ 0.10} & \textbf{4.73 $\pm$ 0.14} & \textbf{4.32 $\pm$ 0.08} \\
\bottomrule
\end{tabular}
\vskip 0.1in
\caption{Error rate (\%) for CIFAR-10.} 
\label{tab:cifar10_vary_sup}
\end{table}

\paragraph{Results with varied label set sizes on SVHN} 
In Table \ref{tab:svhn_vary_sup}, we similarly show results for compared methods of Figure \ref{fig:svhn_vary_sup} and results of methods mentioned above. Fully supervised learning using 73,257 examples achieves an error rate of 2.28 and 2.84  with or without RandAugment. The performance of the baseline models are reported by MixMatch~\cite{berthelot2019mixmatch}. Our reimplementation of MixMatch also resulted in comparable but higher error rates than the reported ones. 
\begin{table}[H]
\centering
\small
\begin{tabular}{l|ccccc}
\toprule
\textbf{Methods / \# Sup} & 250 & 500 & 1,000 & 2,000 & 4,000 \\
\midrule
Pseudo-Label & 21.16 $\pm$ 0.88 & 14.35 $\pm$ 0.37 & 10.19 $\pm$ 0.41 & 7.54 $\pm$ 0.27 & 5.71 $\pm$ 0.07 \\
$\Pi$-Model & 17.65 $\pm$ 0.27 & 11.44 $\pm$ 0.39 & 8.60 $\pm$ 0.18 & 6.94 $\pm$ 0.27 & 5.57 $\pm$ 0.14 \\
Mean Teacher & 6.45 $\pm$ 2.43 & 3.82 $\pm$ 0.17 & 3.75 $\pm$ 0.10 & 3.51 $\pm$ 0.09 & 3.39 $\pm$ 0.11 \\
VAT & 8.41 $\pm$ 1.01 & 7.44 $\pm$ 0.79 & 5.98 $\pm$ 0.21 & 4.85 $\pm$ 0.23 & 4.20 $\pm$ 0.15 \\
MixMatch & 3.78 $\pm$ 0.26 & 3.64 $\pm$ 0.46 & 3.27 $\pm$ 0.31 & 3.04 $\pm$ 0.13 & 2.89 $\pm$ 0.06 \\
UDA (RandAugment) & \textbf{2.72 $\pm$ 0.40}  & \textbf{2.27 $\pm$ 0.09} & \textbf{2.23 $\pm$ 0.07} & \textbf{2.20 $\pm$ 0.06} & \textbf{2.28 $\pm$ 0.10} \\
\bottomrule
\end{tabular}
\caption{Error rate (\%) for SVHN.} 
\label{tab:svhn_vary_sup}
\end{table}

\begin{figure}[h!]
    \centering
    \begin{subfigure}[t]{0.45\textwidth}
        \centering
    \includegraphics[width=\textwidth]{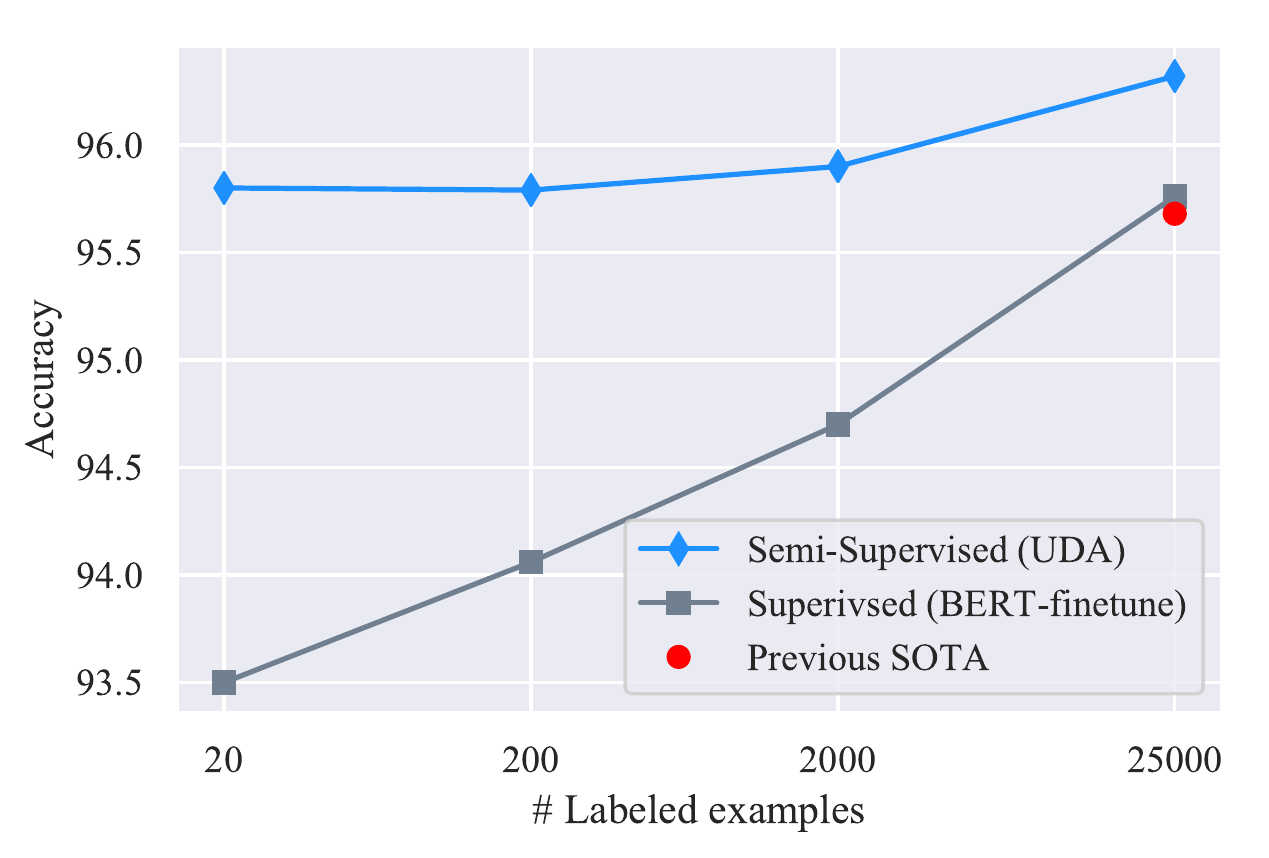}
    \caption{IMDb}
    \end{subfigure}%
    ~ 
    \begin{subfigure}[t]{0.45\textwidth}
        \centering
	\includegraphics[width=\textwidth]{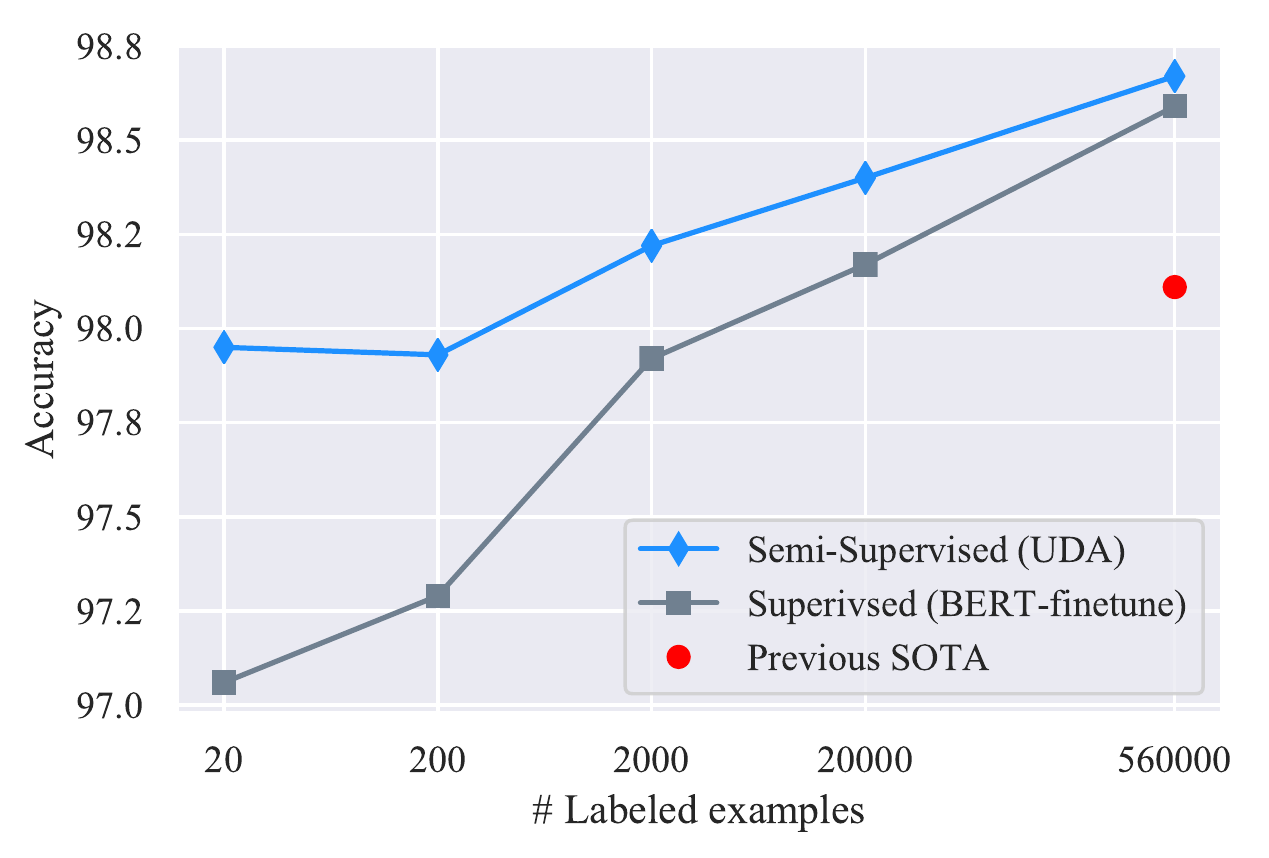}
	\caption{Yelp-2}
    \end{subfigure}
\caption{Accuracy on IMDb and Yelp-2 with different number of labeled examples. In the large-data regime, with the full training set of IMDb, UDA also provides robust gains.} 
\label{fig:text_vary_sup}
\end{figure}

\paragraph{Experiments on Text Classification with Varied Label Set Sizes}
We also try different data sizes on text classification tasks . As show in
Figure \ref{fig:text_vary_sup}, UDA leads to consistent improvements across all labeled data sizes on IMDb and Yelp-2.

\section{Proof for Theoretical Analysis}
\label{sec:proof}
Here, we provide a full proof for Theorem \ref{thm:uda_label_requirement}.
\begin{customthm}{1}
Under UDA, let $Pr(\mathcal{A})$ denote the probability that the algorithm \textit{cannot} infer the label of a new test example given $m$ labeled examples from $P_{L}$. $Pr(\mathcal{A})$ is given by
\[
    Pr(\mathcal{A})= \sum_{i} P_i(1 - P_i)^m.
\]
In addition,  $O(k / \epsilon)$ labeled examples can guarantee an error rate of $O(\epsilon)$, i.e.,
\[
    m = O(k/\epsilon) \implies Pr(\mathcal{A}) = O(\epsilon).
\]
\end{customthm}

\begin{proof}
Let $x'$ be the sampled test example. Then the probability of event $\mathcal{A}$ is

$$Pr(\mathcal{A})= \sum_{i} Pr(\mathcal{A}\mathrm{\ and\ }x' \in C_i)= \sum_{i} P_i(1 - P_i)^m$$

To bound the probability, we would like to find the maximum value of $\sum_{i} P_i(1 - P_i)^m$. We can define the following optimization function: 
\begin{align*}
 \min_{P}& -\sum_{c_i} P_i(1 - P_i)^m   \\
 \mathrm{s.t.} & \sum_{c_i} P_i=1
\end{align*}

The problem is a convex optimization problem and we can construct its the Lagrangian dual function:
$$\mathcal{L}=\sum_{i} P_i(1 - P_i)^m - \lambda (\sum_{i} P_i -1)$$
Using the KKT condition, we can take derivatives to $P_i$ and set it to zero. Then we have
$$\lambda=(1-mP_i) (1-P_i)^{m-1}$$
Hence $P_i=P_j$ for any $i\neq j$. Using the fact that $\sum_i P_i=1$, we have $$P_i=\frac{1}{k}$$ 

Plugging the result back into $Pr(\mathcal{A})= \sum_{i} P_i(1 - P_i)^m$, we have 
$$Pr(\mathcal{A})\leq (1-\frac{1}{k})^m=\exp(m\log(1-\frac{1}{k}))\leq \exp(-\frac{m}{k})$$ 

Hence  when $m=O(\frac{k}{\epsilon})$, we have $$Pr(\mathcal{A})=O(\epsilon)$$

\end{proof}

\section{Extended Related Work}
\label{sec:apdx_related_work}
\textbf{Semi-supervised Learning.} Due to the long history of semi-supervised learning (SSL), we refer readers to~\cite{chapelle2009semi} for a general review.
More recently, many efforts have been made to renovate classic ideas into deep neural instantiations.
For example, graph-based label propagation~\cite{zhu2003semi} has been extended to neural methods via graph embeddings~\cite{weston2012deep,yang2016revisiting} and later graph convolutions~\cite{kipf2016semi}.
Similarly, with the variational auto-encoding framework and reinforce algorithm, classic graphical models based SSL methods with target variable being latent can also take advantage of deep architectures~\cite{kingma2014semi,maaloe2016auxiliary,yang2017semi}.
Besides the direct extensions, it was found that training neural classifiers to classify out-of-domain examples into an additional class~\cite{salimans2016improved} works very well in practice.
Later, Dai et al.~\cite{dai2017good} shows that this can be seen as an instantiation of low-density separation. 

Apart from enforcing consistency on the noised input examples and the hidden representations, another line of research enforces consistency under different model parameters, which is complementary to our method. For example, Mean Teacher~\cite{tarvainen2017mean} maintains a teacher model with parameters being the ensemble of a student model's parameters and enforces the consistency between the predictions of the two models. Recently, fast-SWA~\cite{athiwaratkun2018there} improves Mean Teacher by encouraging the model to explore a diverse set of plausible parameters. In addition to parameter-level consistency, SNTG~\cite{luo2018smooth} also enforces input-level consistency by constructing a similarity graph between unlabeled examples.

\textbf{Data Augmentation.} Also related to our work is the field of data augmentation research.
Besides the conventional approaches and two data augmentation methods mentioned in Section \ref{sec:sda}, a recent approach MixUp~\cite{zhang2017mixup} goes beyond data augmentation from a single data point and performs interpolation of data pairs to achieve augmentation. Recently,  it has been shown that data augmentation can be regarded as a kind of explicit regularization methods similar to Dropout~\cite{hernandez2018data}. 

\textbf{Diverse Back Translation.} Diverse paraphrases generated by back-translation has been a key component in the significant performance improvements in our text classification experiments. We use random sampling instead of beam search for decoding similar to \cite{edunov2018understanding}. There are also recent works on generating diverse translations~\cite{he2018sequence, shen2019mixture, kool2019stochastic} that might lead to further improvements when used as data augmentations.

\textbf{Unsupervised Representation Learning.} Apart from semi-supervised learning, unsupervised representation learning offers another way to utilize unsupervised data. 
Collobert and Weston \cite{collobert2008unified} demonstrated that word embeddings learned by language modeling can improve the performance significantly on semantic role labeling.
Later, the pre-training of word embeddings was simplified and substantially scaled in Word2Vec~\cite{mikolov2013distributed} and Glove~\cite{pennington2014glove}.
More recently,  pre-training using language modeling and denoising auto-encoding has been shown to lead to significant improvements on many tasks in the language domain~\cite{dai2015semi,peters2018deep, radford2018improving,howard2018universal,devlin2018bert}. There is also a growing interest in self-supervised learning for vision~\cite{zhai2019s, henaff2019data, trinh2019selfie}.  

\textbf{Consistency Training in Other Domains.} Similar ideas of consistency training has also been applied in other domains. For example, recently, enforcing adversarial consistency on unsupervised data has also been shown to be helpful in adversarial robustness~\cite{stanforth2019labels,  zhai2019adversarially, najafi2019robustness, carmon2019unlabeled}.
Enforcing consistency w.r.t data augmentation has also been shown to work well for representation learning~\cite{hu2017learning, ye2019unsupervised}. Invariant representation learning ~\cite{liang2018learning,Salazar2018Invariant} applies the consistency loss not only to the predicted distributions but also to representations and has been shown significant improvements on speech recognition.

\section{Experiment Details}
\label{sec:exp_details}
\subsection{Text Classifications}
\textbf{Datasets.} In our semi-supervised setting, we randomly sampled labeled examples from the full supervised set\footnote{\url{http://bit.ly/2kRWoof}, \  \url{https://ai.stanford.edu/~amaas/data/sentiment/}} and use the same number of examples for each category. For unlabeled data, we use the whole training set for DBPedia, the concatenation of the training set and the unlabeled set for IMDb and  external data for Yelp-2, Yelp-5, Amazon-2 and Amazon-5 \cite{mcauley2015image}\footnote{\url{https://www.kaggle.com/yelp-dataset/yelp-dataset}, \url{http://jmcauley.ucsd.edu/data/amazon/}}. Note that for Yelp and Amazon based datasets, the label distribution of the unlabeled set might not match with that of labeled datasets since there are different number of examples in different categories. Nevertheless, we find it works well to use all the unlabeled data.

\textbf{Preprocessing.} 
We find the sequence length to be an important factor in achieving good performance. For all text classification datasets, we truncate the input to 512 subwords since BERT is pretrained with a maximum sequence length of 512.  Further, when the length of an example is greater than 512, we keep the last 512 subwords instead of the first 512 subwords as keeping the latter part of the sentence lead to better performances on IMDb.

\textbf{Fine-tuning BERT on in-domain unsupervised data.} We fine-tune the BERT model on in-domain unsupervised data using the code released by BERT. We try learning rate of 2e-5, 5e-5 and 1e-4, batch size of 32, 64 and 128 and number of training steps of 30k, 100k and 300k. We pick the fine-tuned models by the BERT loss on a held-out set instead of the performance on a downstream task. 

\textbf{Random initialized Transformer.} For the experiments with randomly initialized Transformer, we adopt hyperparameters for BERT base except that we only use 6 hidden layers and 8 attention heads. We also increase the dropout rate on the attention and the hidden states to 0.2, When we train UDA with randomly initialized architectures, we train UDA for 500k or 1M steps on Amazon-5 and Yelp-5 where we have abundant unlabeled data.

\textbf{BERT hyperparameters.} Following the common BERT fine-tuning procedure, we keep a dropout rate of 0.1, and try learning rate of 1e-5, 2e-5 and 5e-5 and batch size of 32 and 128. We also tune the number of steps ranging from 30 to 100k for various data sizes.

\textbf{UDA hyperparameters.} We set the weight on the unsupervised objective $\lambda$ to 1 in all of our experiments. We use a batch size of 32 for the supervised objective since 32 is the smallest batch size on v3-32 Cloud TPU Pod. We use a batch size of 224 for the unsupervised objective when the Transformer is initialized with BERT so that the model can be trained on more unlabeled data. We find that generating one augmented example for each unlabeled example is enough for \bertft.

All experiments in this part are performed on a v3-32 Cloud TPU Pod.

\subsection{Semi-supervised learning benchmarks CIFAR-10 and SVHN}   
\textbf{Hyperparameters for Wide-ResNet-28-2.} 
We train our model for 500K steps. We apply Exponential Moving Average to the parameters with a decay rate of 0.9999. We use a batch size of 64 for labeled data and a batch size of 448 for unlabeled data. The softmax temperature $\tau$ is set to 0.4. The confidence threshold $\beta$ is set to 0.8. We use a cosine learning rate decay schedule: $\cos(\frac{7t}{8T} * \frac{\pi}{2})$ where $t$ is the current step and $T$ is the total number of steps. We use a SGD optimizer with nesterov momentum with the momentum hyperparameter set to 0.9. 
In order to reduce training time, we generate augmented examples before training and dump them to disk. For CIFAR-10, we generate 100 augmented examples for each unlabeled example. Note that generating augmented examples in an online fashion is always better or as good as using dumped augmented examples since the model can see different augmented examples in different epochs, leading to more diverse samples. We report the average performance and the standard deviation for 10 runs.
Experiments in this part are performed on a Tesla V100 GPU.

\textbf{Hyperparameters for Shake-Shake and PyramidNet.} For the experiments with Shake-Shake, we train UDA for 300k steps and use a batch size of 128 for the supervised objective and use a batch size of 512 for the unsuperivsed objective. For the experiments with PyramidNet+ShakeDrop, we train UDA for 700k steps and use a batch size of 64 for the supervised objective and a batch size of 128 for the unsupervised objective. For both models, we use a learning rate of 0.03 and use a cosine learning decay with one annealing cycle following AutoAugment. Experiments in this part are performed on a v3-32 Cloud TPU v3 Pod. 

\subsection{ImageNet} 

\textbf{10\% Labeled Set Setting.} Unless otherwise stated, we follow the standard hyperparameters used in an open-source implementation of ResNet.\footnote{https://github.com/tensorflow/tpu/tree/master/models/official/resnet} 
For the 10\% labeled set setting, we use a batch size of 512 for the supervised objective and a batch size of 15,360 for the unsupervised objective. We use a base learning rate of 0.3 that is decayed by 10 for four times and set the weight on the unsupervised objective $\lambda$ to 20. We mask out unlabeled examples whose highest probabilities across categories are less than 0.5 and set the Softmax temperature to 0.4. The model is trained for 40k steps. Experiments in this part are performed on a v3-64 Cloud TPU v3 Pod. 

\textbf{Full Labeled Set Setting.} For experiments on the full ImageNet, we use a batch size of 8,192 for the supervised objective and a batch size of 16,384 for the unsupervised objective. The weight on the unsupervised objective $\lambda$ is set to 1. 
We use entropy minimization to sharpen the prediction.
We use a base learning rate of 1.6 and decay it by 10 for four times. Experiments in this part are performed on a v3-128 Cloud TPU v3 Pod.

\newpage


\end{document}